\documentclass{article}

\newif\ifpreprint
\preprinttrue

\usepackage{microtype}
\usepackage{graphicx}
\usepackage{booktabs}
\usepackage{array}
\usepackage{float}
\usepackage{hyperref}

\ifpreprint
  \usepackage[preprint]{icml2027}
\else
  \usepackage{icml2027}
\fi
\usepackage{amsmath}
\usepackage{amssymb}
\usepackage{mathtools}
\usepackage{amsthm}
\usepackage{fancyvrb}
\usepackage{algorithm}
\usepackage{algorithmic}
\usepackage[capitalize,noabbrev]{cleveref}
\theoremstyle{plain}
\newtheorem{proposition}{Proposition}
\newtheorem{theorem}{Theorem}
\newtheorem{corollary}{Corollary}
\crefname{corollary}{Corollary}{Corollaries}
\newtheorem{assumption}{Assumption}
\crefname{assumption}{Assumption}{Assumptions}
\crefname{proposition}{Proposition}{Propositions}
\crefname{theorem}{Theorem}{Theorems}

\ifpreprint
\makeatletter
\renewcommand{\printAffiliationsAndNotice}[1]{\global\icml@noticeprintedtrue%
  \stepcounter{@affiliationcounter}%
  {\let\thefootnote\relax\footnotetext{\hspace*{-\footnotesep}\ificmlshowauthors #1\fi%
      \forloop{@affilnum}{1}{\value{@affilnum} < \value{@affiliationcounter}}{
        \textsuperscript{\arabic{@affilnum}}\csname @affilname\the@affilnum\endcsname
      }.%
      \ifdefined\icmlcorrespondingauthor@text
         { }Correspondence to: \icmlcorrespondingauthor@text.
      \fi

      \ \\
      \Notice@String
    }
  }
}
\renewcommand{\ICML@preprint}{\textit{Preprint.}}
\makeatother
\fi

\newcommand{\admitor}{\textsc{AdmitOR}}
\newcommand{\that}{\hat{\theta}}

\icmltitlerunning{Label-Free Certification of LLM-Generated Optimization Models}

\begin{document}

\twocolumn[
  \icmltitle{Label-Free Certification of LLM-Generated Optimization Models: \\
    Agreement Across Model Families and Data Verification}

  \begin{icmlauthorlist}
    \icmlauthor{Junbo Jacob Lian}{nus,thrice}
    \icmlauthor{Huiling Chen}{wzu}
    \icmlauthor{Hanzhang Qin}{nus}
    \icmlauthor{Chung-Piaw Teo}{nus}
  \end{icmlauthorlist}

  \icmlaffiliation{nus}{Institute of Operations Research and Analytics, National University of Singapore, Singapore}
  \icmlaffiliation{thrice}{Thrice AI, Zhejiang, China}
  \icmlaffiliation{wzu}{College of Computer Science and Artificial Intelligence, Wenzhou University, Wenzhou, China}

  \ifpreprint
  \else
    \icmlcorrespondingauthor{Anonymous Author}{anon.email@domain.com}
  \fi

  \icmlkeywords{optimization modeling, large language models, verification, selective prediction, benchmark reliability}

  \vskip 0.3in
]

\printAffiliationsAndNotice{}

\begin{abstract}
A large language model can turn a written problem description into an
optimization model, and a solver then returns a value. Without a reference
answer, the user cannot tell whether this value is the optimum of the problem
that was described. We study certification by agreement: one model extracts
the problem data, models from three families each write a program that reads
these data, and the value is certified when the three programs return it. We
show that the error rate of such a certificate equals the share of data errors
up to the odds of a shared modeling error. More families lower the second
term and leave the first in place, agreement on resampled instances does not
distinguish data errors from correct certificates, and more programs per
family raise the second term.
On 740 problems from OptMATH-Train and NANO-CO, the rule issues 587
certificates, 63 of which miss the benchmark label, and we audit each miss
against the problem text. None is a shared modeling error. In 30, the programs
agree on extracted data that the text does not support. A check of the
extracted data by the other two families withdraws all 30, including all seven
on 200 problems drawn after the rule had been set. The remaining 33 are
defects of the benchmark texts and labels. On 311 problems with human-written
texts from IndustryOR and Mamo ComplexLP, the certifier reaches a label
precision of $0.95$ on releases with corrected texts and labels. There, the three families share
an error only where a text leaves a modeling decision open. All three then
resolve it in the same way, and the error disappears once the text states the
decision.
\ifpreprint Code and run records are available at
\url{https://github.com/junbolian/AdmitOR}.\fi
\end{abstract}

\section{Introduction}
\label{sec:intro}

A planner at a utility describes an allocation problem in words: six power
producers, six supply contracts, a capacity for each producer, a minimum
delivery for each contract, and a unit cost for each producer and contract. A
large language model (LLM) writes a mixed-integer program from the
description, a solver runs it, and the planner receives one number. On a
benchmark, this number is compared with a reference answer. In use, there is
no reference answer, and the planner needs to know whether the number is the
optimum of the problem that was described.

The usual substitute for a reference answer is agreement. Voting over sampled
outputs \citep{selfconsistency2023}, juries of several models
\citep{jury2026}, and debate across backbones \citep{agoraopt2026} accept an
answer because separate attempts return it. The idea is older than language
models. N-version programming runs independently written programs and trusts
the value on which they agree \citep{avizienis1985}. The experiment of
\citet{knight1986} found that such programs fail on the same inputs more often
than independence predicts. Two questions follow for optimization modeling.
How often are programs written by different model families wrong in the same
way, and which errors does their agreement leave in place?

The description above answers the second question. It is record 130870 of
OptMATH-Train \citep{optmath2025}. Its text prints four of the 36 unit costs
and states that similar cost structures apply to the other producers and
contracts. A model that extracts the data of this problem supplies the 32
missing costs. Programs written by three model families then return the same
optimal value, on the stated instance and on five resampled instances, and
this value differs from the benchmark label. The programs agree because they
read the same extracted data, and the instance that they solve is not in the
text.

We answer both questions with a model of certificate errors and with an audit.
\admitor{} extracts the data of a problem once, asks models from three
families to write one program each, and certifies the optimal value when the
three programs return it. The model separates two ways in which such a
certificate can be wrong: the programs share a modeling error, or they are
correct for extracted data that the text does not support. We run \admitor{}
on 740 problems from OptMATH-Train and NANO-CO and on 311 problems with
human-written texts, and audit every certificate whose value misses the
benchmark label against the problem text. Our
contributions are as follows.
\begin{itemize}
\setlength{\itemsep}{1pt}
\item We show that the error rate of an agreement certificate equals the share
of data errors up to the odds of a shared modeling error
(\cref{prop:decomp}). Further families lower a bound on the odds
geometrically and leave the share in place (\cref{prop:families}). Resampled
agreement certifies exactly the programs whose value functions coincide almost
everywhere, and whether it withholds a certificate does not depend on whether
the extracted data are correct (\cref{thm:resample}). Further programs per
family raise the odds (\cref{prop:programs}).
\item The audit measures both terms. Among 587 three-family certificates, no
label miss is a shared modeling error, and 30 are data errors. All other
misses are defects of the benchmark: at least one in ten of the sampled
OptMATH-Train problems carries a label that its text does not determine
(\cref{sec:audit}).
\item We add data verification, in which the families that did not perform the
extraction compare the extracted data with the text. It withdraws every
certificate with a data error, and the result holds on 200 problems drawn
after the rule had been set (\cref{sec:verify}).
\item As the model predicts, agreement on resampled instances and a second
program per family do not raise precision (\cref{sec:resample}).
\item On human-written texts, the certifier reaches a label precision of
$0.95$ on corrected releases, and the families share modeling errors only
where the text leaves a decision open (\cref{sec:human}).
\end{itemize}

\section{Related Work}
\label{sec:related}

\paragraph{Checking LLM-generated optimization models.}
Modeling agents such as OptiMUS \citep{optimus2024} check their own
formulations by re-solving, and recent verifiers, semantic checkers, and
behavioral test suites
\citep{optverifier2026,verisimpl2026,sacopt2025,optargus2026,reloop2026}
filter individual generations. The test suites of \citet{reloop2026} and
\citet{falsver2026} perturb the instance of a single model and check its
response, and \citet{falsver2026} prove that no sound and nontrivial
fixed-threshold perturbation tester exists. EquivaMap \citep{equivamap2025}
decides whether two formulations are equivalent, and OptiHive
\citep{optihive2026} selects one solver per problem from a batch of generated
solvers, instances, and tests. These methods examine the formulation or the
code. \admitor{} compares optimal values across model families and checks the
extracted data in a separate step. The same pattern, in which a language model
proposes and a computation verifies, has been used for decomposition, where
predicted groups of constraints are confirmed by a staged search
\citep{jointdw2026}.

\paragraph{Agreement among models and design diversity.}
Self-consistency and its variants select the answer that recurs most often
across sampled reasoning paths
\citep{selfconsistency2023,usc2023,adaptiveconsistency2023}. Jury-style
aggregation \citep{jury2026}, annotation-free routing
\citep{cascal2026,smoothie2024}, and cross-backbone debate
\citep{agoraopt2026} combine the outputs of several models. Models miss their
own errors \citep{kamoi2024selfcorrect} and favor their own outputs
\citep{panickssery2024selfpref}. A judge that is shown a candidate scores its
plausibility unless the judge first commits an answer of its own
\citep{selfplayjudge2026}. In \admitor{}, every family commits a computed
value before any comparison. For software, N-version programming posed the
same question \citep{avizienis1985}. \citet{knight1986} tested its
independence assumption experimentally, \citet{eckhardt1985} modeled
coincident failures through a failure probability that varies with the input,
and \citet{littlewood1989} extended the model to versions developed with
different methods. We adapt this model to agreement on optimal values and to a
stage that all programs share.

\paragraph{Benchmark reliability.}
The benchmarks of this area
\citep{industryor2024,mamo2024,optmath2025,optibench2024}
contain errors. An expert audit that inspects each item by hand finds errors
in at least $54.0\%$ of IndustryOR items and $23.7\%$ of Mamo ComplexLP items
\citep{orsurvey2025}, and \citet{optimind2025} release both benchmarks with
corrected texts and answers. Label errors are also
common in the test sets of machine learning benchmarks in general
\citep{northcutt2021}. Our audit starts from the certificates that disagree
with the label, so it inspects a small part of the items. It covers a training
set, whose defects enter fine-tuning and experience libraries, and two test
sets in their original and corrected releases.

\paragraph{Selective prediction and the use of certificates.}
A certification rule is a selective predictor \citep{geifman2017}, and we
report its coverage and precision. Conformal selection \citep{jincandes2023}
and Learn then Test \citep{ltt2021} choose a threshold from labeled
calibration data, as we do for one resampling rule. Experience libraries of modeling
agents admit solved problems by reference answers
\citep{alphaopt2025,optskills2026,leanllmopt2026}, and model-generated skills
can transfer negatively \citep{huang2026skills}. \Cref{app:downstream} uses
certificates for admission.

\begin{figure*}[t]
\centering
\includegraphics[width=\textwidth]{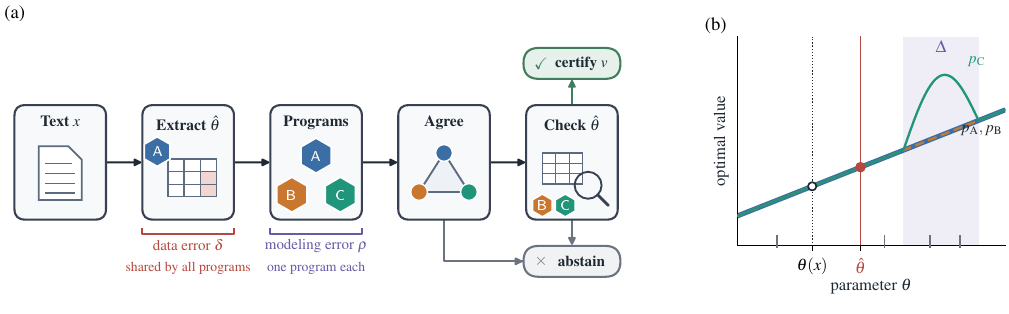}
\caption{(a) The certifier. Family~A extracts the data $\that$ from the text
$x$, each family writes a program $p_f$ that reads $\that$, and the value $v$
is certified when the three programs agree and families B and C find no
difference between $\that$ and $x$. A data error enters before the programs
and is shared by all of them, and a modeling error enters in a single
program. (b) Optimal value as a function of one parameter. Programs
$p_{\mathrm{A}}$ and $p_{\mathrm{B}}$ are correct, and $p_{\mathrm{C}}$ differs
from them on the set $\Delta$. Resampled instances (ticks) that fall in
$\Delta$ separate $p_{\mathrm{C}}$ from the other two. When the extracted
$\that$ differs from the data $\theta(x)$ of the text, all three programs
agree at $\that$, and neither agreement nor resampling reveals the error.}
\label{fig:framework}
\end{figure*}

\begin{figure*}[t]
\centering
\includegraphics[width=0.97\textwidth]{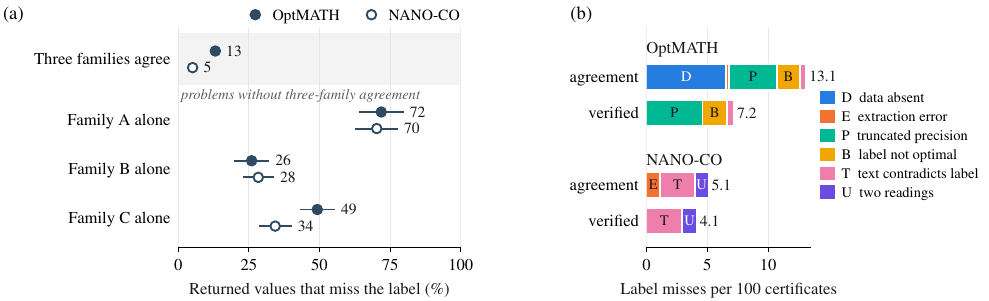}
\caption{(a) Share of returned values that miss the label, with binomial
standard errors. Top row: problems on which the three families agree (411 on
OptMATH, 176 on NANO-CO). Other rows: the value of each family on the problems
without three-family agreement. (b) Label misses per 100 three-family
certificates by audit class, before and after data verification. OptMATH
pools OptMATH-295 and OptMATH-200, and NANO-CO pools NANO-CO-150 and
NANO-CO-95.}
\label{fig:main}
\end{figure*}

\section{Method and Error Model}
\label{sec:method}

\subsection{Extraction, programs, and agreement}
\label{sec:pipeline}

The input is a problem text $x$. An extractor, a model of family~A, returns a
parameter set $\that=E(x)$ that lists each numeric parameter of the problem
with a name, a short description, and its value. For each family
$f\in\{\mathrm{A},\mathrm{B},\mathrm{C}\}$, a model then receives $x$ together
with the names and descriptions of the parameters, but not their values, and
writes a program $p_f$. The program maps a parameter set $\theta$ to the
optimal value $v_f(\theta)$ of the model that it builds, or reports that the
model has none. A program that fails to execute at $\that$ is revised once.
Two values $a$ and $b$ agree when
\begin{equation}
\label{eq:agree}
|a-b|\le 10^{-4}\max(1,|a|,|b|).
\end{equation}
The $k$-family rule certifies a value on which the programs of at least $k$
families agree at $\that$, for $k\in\{2,3\}$, and the single-family rule of
family $f$ certifies $v_f(\that)$ whenever it exists. A problem without a
certificate is an abstention. A certificate is a value that agreement supports, and the word does not
denote a proof. \Cref{fig:framework}(a) shows the certifier with the data
verification of \cref{sec:method-verify}.

\subsection{Certificate errors under agreement}
\label{sec:decomp}

For a text $x$ and a parameter set $\theta$, let $V(x,\theta)$ be the optimal
value of the model that $x$ describes with data $\theta$. When the text
determines its data $\theta(x)$, the stated problem has the optimal value
$v^\star(x)=V(x,\theta(x))$. A program $p_f$ is \emph{correct at} $\theta$ if
$p_f(\theta)=V(x,\theta)$. Consider a rule that requires the agreement of $k$
families, and the following events for a random problem. $A$: the $k$ programs
return a common value $v$ at $\that$. $C$: the $k$ programs are correct at
$\that$. $W=A\setminus C$. $D$: the text does not determine its data, or
$V(x,\that)\neq v^\star(x)$. A program that is correct at $\that$ and one that
is not return different values. On $W$, all $k$ programs are therefore
incorrect at $\that$ and return the same value, a \emph{shared modeling
error}. On $C$, the certified value is $V(x,\that)$, and the certificate is
wrong exactly when $D$ holds, a \emph{data error}. Let
$\mathrm{Err}\subseteq A$ be the event that the certified value is not
$v^\star(x)$, and write $\delta_k=\Pr(D\mid C)$ and $\rho_k=\Pr(W)/\Pr(C)$.

\begin{proposition}
\label{prop:decomp}
The error rate of the certificates of a $k$-family rule satisfies
\begin{equation}
\label{eq:decomp}
\frac{\delta_k}{1+\rho_k}\;\le\;\Pr(\mathrm{Err}\mid A)\;\le\;
\frac{\delta_k+\rho_k}{1+\rho_k}.
\end{equation}
\end{proposition}

The error rate equals the share of data errors $\delta_k$ up to the odds
$\rho_k$ of a shared modeling error. The two terms respond differently to the
design of the rule. To state this, we use the model of coincident failures of
\citet{eckhardt1985}, in which the failure probability of a program varies
with the input. All proofs are in \cref{app:proofs}.

\begin{assumption}
\label{ass:ci}
Conditionally on $(x,\that)$, all programs are independent. A program of
family $f$ is incorrect at $\that$ with probability $\varepsilon_f(x,\that)$,
and an incorrect program returns any given value with probability at most
$\gamma(x,\that)$.
\end{assumption}

\begin{proposition}[Families]
\label{prop:families}
Under \cref{ass:ci}, $\Pr(C)=\mathbb{E}\bigl[\prod_f(1-\varepsilon_f)\bigr]$
and $\Pr(W)\le\mathbb{E}\bigl[\gamma^{k-1}\prod_f\varepsilon_f\bigr]$. If
$\varepsilon_f\le\bar\varepsilon$ and $\gamma\le\bar\gamma$ almost surely,
then
$\rho_k\le\bar\gamma^{-1}\bigl(\bar\gamma\bar\varepsilon/(1-\bar\varepsilon)\bigr)^{k}$.
If $(\varepsilon_f)_f$ is independent of $D$, then $\delta_k=\Pr(D)$ for every
$k$.
\end{proposition}

The odds $\rho_k$ fall geometrically in $k$ when
$\bar\gamma\bar\varepsilon<1-\bar\varepsilon$. The factor $\gamma$ is specific
to value agreement: incorrect programs must coincide in a real-valued optimum,
which is a stronger requirement than failing together. With $\gamma=1$ and
identical $\varepsilon_f$, the bound is the coincident-failure probability
$\mathbb{E}[\varepsilon^k]\ge(\mathbb{E}\,\varepsilon)^k$ of
\citet{eckhardt1985}. The share $\delta_k$ does not fall, because every program
reads the same $\that$. Agreement filters data errors only to the extent that
programs are less often correct on an extraction with a data error.

Further programs of the same families act in the opposite direction. Suppose
that each family writes $j$ programs and that a value is certified when at
least one program of every family returns it.

\begin{proposition}[Programs per family]
\label{prop:programs}
Under \cref{ass:ci}, conditionally on $(x,\that)$, let $T_j$ be the
probability that $V(x,\that)$ is returned by at least one program of every
family, and let $S_j(u)$ be this probability for an incorrect value $u$ that a
program of family $f$ returns with probability $q_f$. Then
\begin{equation}
\label{eq:programs}
\frac{S_j(u)}{T_j}=\frac{S_1(u)}{T_1}\,\prod_f
\frac{\sum_{i=0}^{j-1}(1-q_f)^i}{\sum_{i=0}^{j-1}\varepsilon_f^{\,i}},
\end{equation}
and every factor of the product is at least one when
$q_f+\varepsilon_f\le 1$.
\end{proposition}

Since $q_f\le\gamma\varepsilon_f$, the condition holds whenever
$\varepsilon_f\le 1/2$. A second program per family multiplies the support of
the correct value by $\prod_f(1+\varepsilon_f)$ and the support of every
incorrect value by the larger factor $\prod_f(2-q_f)$. Programs within a family
therefore add coverage and raise the odds of an incorrect value, whereas
further families lower $\rho_k$.

\subsection{Agreement on resampled instances}
\label{sec:method-resample}

Because each program takes $\theta$ as input, the programs can also be compared
away from the stated instance. The extractor proposes a perturbation domain
$\Theta$ around $\that$, and $m$ instances $\theta_1,\dots,\theta_m$ are drawn
from it. A resampled rule certifies $v(\that)$ when the programs of the
required number of families agree at $\that$ and at every drawn instance on
which they return a value (\cref{app:rules}).

\begin{assumption}
\label{ass:milp}
$\Theta\subset\mathbb{R}^n$ is an open box. For every $\theta\in\Theta$, each
program builds a mixed-integer linear program with bounded integer variables
whose coefficients are polynomials in $\theta$, and it returns the optimal
value or the status of this program. The instances are independent draws from
a distribution $Q$ on $\Theta$ with a density that is positive almost
everywhere.
\end{assumption}

\begin{theorem}[Resampling]
\label{thm:resample}
Under \cref{ass:milp}, let $\Delta\subseteq\Theta$ be the set of instances on
which two programs return different outputs. (i) $\Delta$ is semialgebraic,
and either $\Delta$ has Lebesgue measure zero or $\Delta$ contains an open
ball. (ii) In the first case, the two programs agree on all $m$ instances
almost surely. In the second case, $\mu=Q(\Delta)>0$, and the programs agree on
all $m$ instances with probability $(1-\mu)^m$. (iii) Programs that are correct
on $\Theta$ agree on every instance, whether or not $D$ holds.
\end{theorem}

As $m$ grows, resampled agreement therefore certifies exactly the groups of
programs whose value functions coincide almost everywhere on $\Theta$
(\cref{fig:framework}(b)). It
separates incorrect programs that coincide at $\that$ without being the same
model, for example programs that differ in a constraint that is inactive at
$\that$, and it lowers $\rho_k$ by this amount. It cannot separate programs
that build the same incorrect model. By (iii), it treats a data error like a
correct certificate: it withholds one only when the programs also differ away
from $\that$. It
can also withhold a correct certificate, when a program that is correct at
$\that$ is incorrect elsewhere on $\Theta$. When $\rho_k$ is already close to
zero, resampling can lower coverage and cannot raise precision.

\begin{table*}[t]
\caption{Coverage and label precision ($\pm$ binomial standard error) of the
certification rules. \emph{Resampled}: agreement on five resampled instances with a fitted
threshold (\cref{app:rules}). \emph{Verified}: the
certificate is withdrawn when family B or C flags the extracted data.}
\label{tab:rules}
\centering
\small
\setlength{\tabcolsep}{5pt}
\begin{tabular}{lcccccccc}
\toprule
 & \multicolumn{2}{c}{OptMATH-295} & \multicolumn{2}{c}{OptMATH-200} & \multicolumn{2}{c}{NANO-CO-150} & \multicolumn{2}{c}{NANO-CO-95} \\
\cmidrule(lr){2-3}\cmidrule(lr){4-5}\cmidrule(lr){6-7}\cmidrule(lr){8-9}
Rule & Cov. & Precision & Cov. & Precision & Cov. & Precision & Cov. & Precision \\
\midrule
Family A alone & 0.898 & $0.807\pm0.024$ & 0.890 & $0.854\pm0.026$ & 0.887 & $0.820\pm0.033$ & 0.842 & $0.863\pm0.038$ \\
Family B alone & 0.922 & $0.838\pm0.022$ & 0.945 & $0.878\pm0.024$ & 0.993 & $0.893\pm0.025$ & 0.990 & $0.872\pm0.034$ \\
Family C alone & 0.966 & $0.796\pm0.024$ & 0.955 & $0.853\pm0.026$ & 0.993 & $0.886\pm0.026$ & 0.990 & $0.840\pm0.038$ \\
\midrule
Two families agree & 0.925 & $0.835\pm0.022$ & 0.920 & $0.886\pm0.023$ & 0.980 & $0.898\pm0.025$ & 0.979 & $0.871\pm0.035$ \\
Three families agree & 0.824 & $0.856\pm0.022$ & 0.840 & $0.887\pm0.024$ & 0.707 & $0.972\pm0.016$ & 0.737 & $0.914\pm0.034$ \\
Three families, resampled & 0.759 & $0.866\pm0.023$ & 0.795 & $0.893\pm0.025$ & 0.633 & $0.990\pm0.011$ & 0.621 & $0.915\pm0.036$ \\
\midrule
Two families, verified & 0.756 & $0.915\pm0.019$ & 0.780 & $0.929\pm0.021$ & 0.907 & $0.919\pm0.023$ & 0.937 & $0.876\pm0.035$ \\
Three families, verified & 0.681 & $0.930\pm0.018$ & 0.730 & $0.925\pm0.022$ & 0.673 & $0.980\pm0.014$ & 0.716 & $0.926\pm0.032$ \\
\bottomrule
\end{tabular}

\end{table*}

\subsection{Data verification}
\label{sec:method-verify}

Data verification acts on $\delta_k$, the term that agreement and resampling
leave in place. Each verifying family receives the text $x$ and the extracted
parameters $\that$ and returns two lists. The first contains every parameter
whose value differs from the text, appears in a different order, or belongs to
a different entity (\emph{mismatch}). The second contains every parameter
whose value the text does not give, which includes values for which the text
gives only examples or a range (\emph{not stated}). The verifier is asked not
to judge modeling choices (\cref{app:prompts}). A family raises a flag when
either list is not empty. The certificate is withdrawn when family~B or
family~C raises a flag, so the decision rests on the two families that did not
perform the extraction. We also report the stricter rule under which both
must raise a flag.

Let $d$ and $\varphi$ be the probabilities that a rule withdraws a certificate
with and without a data error. Among the certificates with correct programs,
the rule keeps the fraction $1-\delta_k d-(1-\delta_k)\varphi$ and changes the
share of data errors from $\delta_k$ to
\begin{equation}
\label{eq:verify}
\delta_k'=\frac{\delta_k(1-d)}{\delta_k(1-d)+(1-\delta_k)(1-\varphi)}.
\end{equation}
\Cref{prop:decomp} applies to the certificates that a rule keeps, which
gives a bound on their error rate.

\begin{corollary}
\label{cor:verified}
Let a verification rule have the rates $(d,\varphi)$, and let
$\kappa=1-\delta_k d-(1-\delta_k)\varphi$. The error rate of the certificates
that the rule keeps is at most $\delta_k'+\rho_k/\kappa$.
\end{corollary}

A rule is a set of flag patterns on which the certificate is withdrawn. By the
Neyman--Pearson lemma \citep{lehmann2005}, the rules that maximize $d$ for a
given $\varphi$ withdraw the patterns in decreasing order of their likelihood
ratio between certificates with and without a data error. \Cref{sec:verify}
estimates these ratios.

\section{Experiments}
\label{sec:experiments}

\subsection{Setup}
\label{sec:setup}

\paragraph{Models.}
Family~A is DeepSeek-V4-Flash, family~B is GPT-6.1 Sol, and family~C is
Claude Sonnet~5.5, all at temperature zero. Family~A also performs the extraction.
Families A and C write the program directly, and family~B first writes an
intermediate representation of the model. The programs use two modeling
interfaces and two solvers (\cref{app:config}).

\paragraph{Problems.}
OptMATH-Train \citep{optmath2025} contains 201{,}071 problems whose texts were
generated from instances by a language model. We draw 300 problems uniformly
at random, set aside five that were used during development, and call the
remaining set OptMATH-295. NANO-CO \citep{optskills2026} contains 245
combinatorial problems of 49 types. We split it into NANO-CO-150, whose labels
serve to fit the threshold of a resampling rule, and NANO-CO-95. These three
sets are the \emph{development sets}: the verification rule was written after
their label misses had been audited. We then drew 200 further problems from
OptMATH-Train, OptMATH-200, before any model was run on them, and left the
rule and the models unchanged. \Cref{sec:human} adds two benchmarks whose
texts were written by people.

\begin{table*}[t]
\caption{Audit of the certificates that miss the label, by class. OM and NC
abbreviate OptMATH and NANO-CO. \emph{Verified}: three-family certificates that
remain after data verification, all sets. \emph{Two only}: certificates that
two families support and three do not, all sets.}
\label{tab:audit}
\centering
\footnotesize
\setlength{\tabcolsep}{3pt}
\begin{tabular}{@{}llrrrrrrr>{\raggedright\arraybackslash}p{3.95cm}@{}}
\toprule
 & & \multicolumn{6}{c}{Three families agree} & Two only & \\
\cmidrule(lr){3-8}\cmidrule(lr){9-9}
 & Class & OM-295 & OM-200 & NC-150 & NC-95 & All & Verified & All & Example \\
\midrule
D & Data absent from the text & 20 & 7 & 0 & 0 & 27 & 0 & 9 & Record 130870 prints four of its 36 unit costs. \\
E & Extraction error & 1 & 0 & 1 & 1 & 3 & 0 & 1 & Record 166455: a table is read in transposed order. \\
M & Shared modeling error & 0 & 0 & 0 & 0 & 0 & 0 & 3 & CFLP\_42: a distance is weighted by a demand. \\
\midrule
P & Truncated precision in the text & 9 & 7 & 0 & 0 & 16 & 16 & 0 & Record 198916: costs printed with three decimals. \\
B & Label not the optimum of the text & 3 & 5 & 0 & 0 & 8 & 7 & 1 & Record 89312: optimum of the text 2396, label 2482. \\
T & Text contradicts labeled instance & 2 & 0 & 2 & 3 & 7 & 7 & 8 & MC\_114: unit costs in the text, capacities in the label. \\
U & Text admits two readings & 0 & 0 & 0 & 2 & 2 & 2 & 8 & mTSP\_161: one tour, or one tour per vehicle. \\
\midrule
 & Label misses & 35 & 19 & 3 & 6 & 63 & 32 & 30 & \\
 & Certificates & 243 & 168 & 106 & 70 & 587 & 516 & 110 & \\
\bottomrule
\end{tabular}

\end{table*}

\paragraph{Scoring.}
A certified value matches a label when the two agree within relative tolerance
$10^{-4}$ or when rounding the value to the number of decimals printed in the
label reproduces the label, with a variant for the two-decimal labels of
NANO-CO (\cref{app:config}). Coverage is the share of problems with a
certificate, and precision is the share of certificates that match the label.
We report the binomial standard error of each precision.

\paragraph{Audit.}
For every certificate of the two-family or the three-family rule that misses
the label, we compare the extracted data with the text and assign one of the
seven classes of \cref{tab:audit}. Classes D and E are the data errors of
\cref{sec:decomp}, class M is a shared modeling error, and classes P, B, T,
and U are benchmark defects. We assign class D when decisive extracted values
do not occur in the text, and class P when the text prints rounded
coefficients and the certified value lies within $0.1\%$ of the label. For the
other classes, we re-solve the problem from its text with an exact solver or
by enumeration (\cref{app:audit}).

\subsection{Agreement and abstention}
\label{sec:agreement}

\Cref{tab:rules} compares the rules. Three-family agreement certifies between
$71\%$ and $84\%$ of the problems, and its label precision exceeds that of
every single family and of two-family agreement on all four sets. The margin
over the best single family, family~B, is small on OptMATH and larger on
NANO-CO, and the single-family rules cover more problems.

The rules differ only on the problems without three-family agreement, because
on the other problems all three programs return the certified value.
\Cref{fig:main}(a) shows what the abstention of the three-family rule removes.
On the problems where the three families agree, few values miss the label. On
the remaining problems, the value of each single family misses it far more
often, on both benchmarks. A single program gives no signal
that separates the two groups, and the disagreement of the other families
does.

The families are far from independent. All three match the label much more
often than the product of their single-family rates predicts
(\cref{app:rules}), so the probability of a correct program varies across
problems, as in the model of \citet{eckhardt1985}.

\subsection{Audit of the label misses}
\label{sec:audit}

\begin{table}[t]
\caption{Certificates whose label miss is an error of the certifier, in percent of the
certificates of each rule and set, and their number over the four sets by kind
(\emph{Data}: classes D and E; \emph{Model}: class M).}
\label{tab:pipeline}
\centering
\small
\setlength{\tabcolsep}{4.5pt}
\begin{tabular}{@{}lccccccc@{}}
\toprule
 & \multicolumn{2}{c}{OptMATH} & \multicolumn{2}{c}{NANO-CO} & & \multicolumn{2}{c}{Count} \\
\cmidrule(lr){2-3}\cmidrule(lr){4-5}\cmidrule(l){7-8}
Rule & 295 & 200 & 150 & 95 & All & Data & Model \\
\midrule
Two families & 10.3 & 4.9 & 2.0 & 3.2 & 6.2 & 40 & 3 \\
Three families & 8.6 & 4.2 & 0.9 & 1.4 & 5.1 & 30 & 0 \\
\midrule
Two, verified & 0.9 & 0.0 & 0.7 & 2.2 & 0.8 & 2 & 3 \\
Three, verified & 0.0 & 0.0 & 0.0 & 0.0 & 0.0 & 0 & 0 \\
\bottomrule
\end{tabular}

\end{table}

\Cref{tab:audit} gives the classes of the 63 misses among the 587 three-family
certificates, with one example per class, and \cref{tab:pipeline} restates the
audit in the terms of \cref{prop:decomp}. No miss is a shared modeling error,
and the one-sided $95\%$ upper confidence bound on the rate of the event $W$
is $0.51\%$. Shared modeling errors appear only among the certificates that
two families support and three do not, together with data errors and
benchmark defects. The disagreement of the third family therefore most often
marks a text that is incomplete or defective. As \cref{prop:families}
predicts, the third family removes the shared modeling errors and leaves the
share of data errors almost unchanged (\cref{app:audit}).

The certifier errors under three-family agreement are data errors. Most of
them arise from texts that give values for some entities and state that
similar values apply to the others, or that give only a range, as in the
example of \cref{sec:intro}. The extractor supplies the missing values. In the
others, the extraction differs from the text, for example a demand table
stored by market and product and read by product and market. In each case, the
three programs return the same value because they read the same $\that$.

The other 33 misses are benchmark defects, and in each of them the certified
value is supported by the text. The text prints rounded coefficients (P), our
exact re-solve of a complete text reproduces the certified value (B), the
labeled solution breaks a rule that the text states (T), or the text supports
two readings (U). Counting the two-family certificates as well, at least
$13\%$ of the sampled OptMATH-Train problems and $9\%$ of the NANO-CO problems
carry a label that their text does not determine. These shares are lower
bounds, because we audited only problems whose certificate misses the label.

\subsection{Data verification}
\label{sec:verify}

\begin{table}[t]
\caption{Data verification on three-family certificates. \emph{Targeted}:
number of certificates with a data error (classes D and E) that the flag
withdraws, out of 23 on the development sets and seven on OptMATH-200.
\emph{Matching}: share of the label-matching certificates that the flag
withdraws.}
\label{tab:check}
\centering
\small
\setlength{\tabcolsep}{2.5pt}
\begin{tabular}{@{}lcccc@{}}
\toprule
 & \multicolumn{2}{c}{Development sets} & \multicolumn{2}{c}{OptMATH-200} \\
\cmidrule(lr){2-3}\cmidrule(lr){4-5}
Flag raised by & Targeted & Matching (\%) & Targeted & Matching (\%) \\
\midrule
Family A & 20 & 5.1 & 7 & 8.1 \\
Family B & 23 & 4.8 & 7 & 6.0 \\
Family C & 22 & 4.8 & 7 & 6.0 \\
\midrule
B or C & 23 & 6.9 & 7 & 9.4 \\
B and C & 22 & 2.7 & 7 & 2.7 \\
\bottomrule
\end{tabular}

\end{table}

\begin{figure}[t]
\centering
\includegraphics[width=\columnwidth]{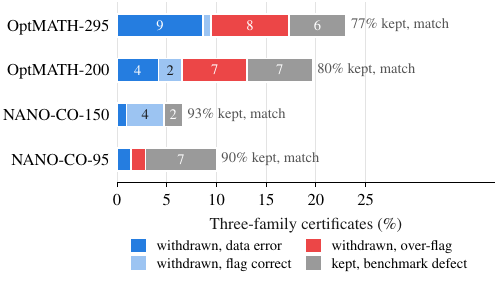}
\caption{Outcome of data verification for the three-family certificates of
each set, in percent of the certificates. The remaining certificates are kept
and match the label.}
\label{fig:verify}
\end{figure}

\Cref{tab:check} reports data verification for the 30 certificates with a data
error. The rule that withdraws a certificate when family B or C raises a flag
removes all of them. On OptMATH-200, we classified the label misses before we
looked at the flags, and the rule withdraws all seven data errors there, together
with one benchmark defect. In the terms of \cref{eq:verify}, the detection
rate $d$ has a one-sided $95\%$ lower confidence bound of $0.90$. The
not-stated list of the reply catches the values that the text does not give,
and the mismatch list catches the values that it gives differently. After
verification, every remaining label miss is a benchmark defect
(\cref{fig:main}(b)), and label precision rises on all four sets
(\cref{tab:rules}).

Verification lowers coverage, because it also withdraws a small share of the
label-matching certificates. We reviewed each of them
(\cref{fig:verify}). About a quarter of these flags are correct: the
extraction differs from the text without changing the optimum, or the text
omits data and the certified value still matches the label. The others are
over-flags, mostly on a complete list that the text introduces with ``for
example'', on a constant that follows from the text, or on the diagonal of a
matrix.

\Cref{cor:verified} turns these measurements into a bound. With the estimated
share of data errors and rate of over-flags, and with the $95\%$ bounds on $d$
and $\rho_3$, the error rate of the verified three-family certificates is at
most $1.1\%$ (\cref{app:verify}). In the sample, no verified certificate has
an audited certifier error.

The flag patterns order the rules as \cref{sec:method-verify} describes
(\cref{tab:joint}). Families B and C flag almost every data error together,
and they rarely flag the same label-matching certificate. The stricter rule,
under which both must raise a flag, withdraws fewer correct certificates and
misses one data error, the transposed demand table. The extractor's own family
is the weakest verifier of its extraction.

Verification does not replace the third family. Among the certificates that
only two families support, it withdraws most data errors and none of the
shared modeling errors (\cref{tab:pipeline}). The verified two-family rule
covers more problems, at similar label precision on OptMATH and lower
precision on NANO-CO (\cref{tab:rules}). Verification costs two model calls
per problem, less than half the cost of certification (\cref{app:cost}).

\subsection{Resampling and programs per family}
\label{sec:resample}

\Cref{thm:resample} predicts that agreement on resampled instances does not
distinguish data errors from correct certificates and can only lower a term
$\rho_3$ that the audit already finds at zero. The measurements agree. With
five resampled instances and a threshold fitted on NANO-CO-150, the resampled
rule withholds a small share of the three-family certificates on OptMATH,
keeps 25 of the 30 data errors, and changes precision by less than one
standard error (\cref{tab:rules}). The five data errors that it withholds are
certificates whose programs also disagree on a resampled instance. Under two-family
agreement, where $\rho_2$ is positive, resampling withholds one of the three
shared modeling errors.

\Cref{prop:programs} predicts that a second program per family raises coverage
and not precision. With two programs per family on OptMATH-295 and NANO-CO-95,
three-family agreement covers more problems at the same precision
(\cref{tab:k6-full}), while the cost of certification nearly doubles.

\subsection{Human-written texts}
\label{sec:human}

\begin{table}[t]
\caption{IndustryOR and Mamo ComplexLP, pooled: 311 problems with the
published labels and the 309 problems of the corrected releases.}
\label{tab:human}
\centering
\footnotesize
\setlength{\tabcolsep}{2.5pt}
\begin{tabular}{lcccc}
\toprule
 & \multicolumn{2}{c}{Original} & \multicolumn{2}{c}{Corrected} \\
\cmidrule(lr){2-3}\cmidrule(lr){4-5}
Rule & Cov. & Precision & Cov. & Precision \\
\midrule
Family B & 0.929 & $0.540\pm0.029$ & 0.987 & $0.911\pm0.016$ \\
Two families & 0.865 & $0.565\pm0.030$ & 0.942 & $0.914\pm0.016$ \\
Three families & 0.666 & $0.556\pm0.035$ & 0.725 & $0.946\pm0.015$ \\
Three, resampled & 0.559 & $0.534\pm0.038$ & 0.618 & $0.937\pm0.018$ \\
Three, verified & 0.543 & $0.580\pm0.038$ & 0.628 & $0.954\pm0.015$ \\
\bottomrule
\end{tabular}

\end{table}

IndustryOR \citep{industryor2024} and Mamo ComplexLP \citep{mamo2024} contain
311 problems whose texts were written by people. The expert audit of
\citet{orsurvey2025} removes about half of them, and \citet{optimind2025}
release both benchmarks with corrected texts and answers. We ran the certifier
unchanged on both releases and audited every three-family label miss
(\cref{app:human}). Against the original labels, three-family agreement
reaches a precision of $0.55$, and most of its misses fall on problems that
the expert audit removes. Against the corrected answers, it reaches $0.95$
(\cref{tab:human}), and close to half of its remaining misses are answers
that are not the optimum of the stated problem.

The corrected releases also locate the modeling errors that the families
share. Twenty-six diet problems of Mamo ComplexLP ask for the cheapest
combination of foods and do not say whether a fraction of a portion may be
bought. On each of them, the three families declare continuous quantities and
agree on a value that misses the label, the optimum with whole portions. On
the published text, these certificates are of class U, since the text supports
both readings. The corrected release adds one sentence to each text that states
the intended reading: ``You could only assign integral portions of each type of
food.'' Under this reading, the 26 certificates of the published text are
shared modeling errors, the event $W$ of \cref{sec:decomp}: all three programs
are incorrect and return the same value. With the sentence, the three
programs return the optimum with integer portions on every one of these
problems, and every certificate matches the answer. This is the mechanism of
\citet{eckhardt1985}: the probability of an error depends on the input, here
on whether the text states the decision, and it is high for all families on
the same inputs.

On the corrected releases, data errors are rarer than on the generated texts.
Verification withdraws all of them, and its effect on precision stays within
one standard error.

\section{Discussion}
\label{sec:discussion}

\paragraph{Where model families fail together.}
The independence assumption of N-version programming fails for human
programmers \citep{knight1986}, and language models share training data. In
our measurements, the coincident errors of three families enter through the
input that all programs share. On generated texts, this input is the extracted
data. On human-written texts, it is also a modeling decision that the text
leaves open, which neither agreement nor verification detects. Given the same
data and a text that states its decisions, the three programs did not arrive
at the same wrong model in any audited certificate. Diversity is therefore
most useful at the shared stages, and data verification is the corresponding
step for the extraction.

\paragraph{Choice of rule.}
The verified three-family rule leaves no audited certifier error. The
verified two-family rule covers more problems and keeps the shared modeling
errors, which no check of the data can remove (\cref{tab:pipeline}). The third
family removes the shared modeling errors, and verification removes the data
errors. Where data errors
are rare, as on the corrected human-written releases, the stricter
verification rule keeps more certificates at the same precision. Further
programs of the same families raise coverage only.

\paragraph{Label precision and defective benchmarks.}
On IndustryOR and Mamo ComplexLP, the same certifier reaches a label precision
of $0.55$ on the published releases and $0.95$ on the corrected ones. On benchmarks with
defects of this frequency, label precision understates the reliability of a
certifier and cannot rank certifiers whose differences are smaller than the
defect rate. An audit of the label misses is feasible because the misses are
few. At least $13\%$ of the sampled OptMATH-Train problems have a label that
the text does not determine, and certificates can filter such training sets
(\cref{app:downstream}).

\paragraph{Scope.}
The measurements hold for three specific models. The bound on shared
modeling errors holds for the generated texts of OptMATH-Train and NANO-CO and
assumes that a certificate that matches the label is correct. On human-written
texts, shared modeling errors occur where a text leaves a decision open. The certifier
abstains on a quarter to a third of the problems, and close to half of the
certificates that verification withdraws are over-flags.

\section{Conclusion}
\label{sec:conclusion}

The error rate of an agreement certificate is the share of data errors up to
the odds of a shared modeling error, and the design choices act on one term
each. On 740 problems with generated texts, three model families share no
modeling error, and verification by the two other families withdraws every
data error, including those on 200 problems drawn afterwards. On human-written
texts, they share one only where the text leaves a decision open, and label
precision reaches $0.95$ on corrected releases. Resampling and further
programs per family add cost without precision.

\section*{AI Use Statement}

We used AI tools to assist with programming, namely writing, refactoring, and
debugging the research code, and with editing the manuscript for readability.
The authors reviewed all AI-assisted code and text and take full
responsibility for the content of the paper.

\section*{Impact Statement}

This work aims to make optimization models generated by language models safer
to use by stating when their output is supported and when the system should
abstain. A certificate is evidence and not a proof, and decisions with large
consequences should not rest on it alone. The audit also documents defects in
public benchmarks, which we report so that they can be corrected.

\bibliography{refs}
\bibliographystyle{icml2027}

\newpage
\appendix
\onecolumn
\raggedbottom
\section{Proofs}
\label{app:proofs}

\paragraph{Proof of \cref{prop:decomp}.}
The events $C$ and $W$ partition $A$, so
$\Pr(A)=\Pr(C)+\Pr(W)=\Pr(C)(1+\rho_k)$. On $C$, the certified value is
$V(x,\that)$. If $D$ does not hold, the text determines its data and
$V(x,\that)=v^\star(x)$, so the certificate is correct. If $D$ holds, either
$V(x,\that)\neq v^\star(x)$ or the text determines no $v^\star(x)$, so the
certificate is wrong. Hence $\mathrm{Err}\cap C=D\cap C$ and
\[
\Pr(D\cap C)\;\le\;\Pr(\mathrm{Err})\;\le\;\Pr(D\cap C)+\Pr(W).
\]
Dividing by $\Pr(A)$ and using $\Pr(D\cap C)=\delta_k\Pr(C)$ and
$\Pr(W)=\rho_k\Pr(C)$ gives \cref{eq:decomp}. \hfill$\square$

\paragraph{Proof of \cref{prop:families}.}
Condition on $(x,\that)$. By \cref{ass:ci}, the $k$ programs are independent
and program $f$ is correct at $\that$ with probability $1-\varepsilon_f$, so
the conditional probability of $C$ is $\prod_f(1-\varepsilon_f)$, and taking
expectations gives $\Pr(C)$. On $W$, all $k$ programs are incorrect at $\that$
and return one value. The conditional probability that all are incorrect is
$\prod_f\varepsilon_f$. Given that they are incorrect, let $g_f(u)$ be the
probability that program $f$ returns the value $u$, with $g_f(u)\le\gamma$.
The programs return a common value with probability
\[
\sum_u\prod_{f}g_f(u)\;\le\;\gamma^{k-1}\sum_u g_{f_1}(u)\;=\;\gamma^{k-1},
\]
where $f_1$ is one of the families. Hence the conditional probability of $W$ is
at most $\gamma^{k-1}\prod_f\varepsilon_f$, and taking expectations gives the
bound on $\Pr(W)$. Under the uniform bounds,
$\Pr(W)\le\bar\gamma^{k-1}\bar\varepsilon^{k}$ and
$\Pr(C)\ge(1-\bar\varepsilon)^{k}$, and the ratio is the stated bound on
$\rho_k$. Finally,
$\delta_k=\mathbb{E}[\mathbf{1}_D\prod_f(1-\varepsilon_f)]/\mathbb{E}[\prod_f(1-\varepsilon_f)]$,
which equals $\Pr(D)$ when $(\varepsilon_f)_f$ is independent of $D$.
\hfill$\square$

In general, $\delta_k=\Pr(D)\,a_1/\bigl(\Pr(D)\,a_1+(1-\Pr(D))\,a_0\bigr)$ with
$a_1=\mathbb{E}[\prod_f(1-\varepsilon_f)\mid D]$ and
$a_0=\mathbb{E}[\prod_f(1-\varepsilon_f)\mid D^{c}]$. Agreement lowers the share
of data errors below $\Pr(D)$ only when $a_1<a_0$, that is, when the programs
are less often jointly correct on an extraction with a data error.

\paragraph{Proof of \cref{prop:programs}.}
Condition on $(x,\that)$. A program of family $f$ returns $V(x,\that)$ with
probability $1-\varepsilon_f$, so at least one of its $j$ independent programs
does so with probability $1-\varepsilon_f^{\,j}$, and
$T_j=\prod_f(1-\varepsilon_f^{\,j})$. In the same way,
$S_j(u)=\prod_f\bigl(1-(1-q_f)^j\bigr)$. The identity
$1-a^j=(1-a)\sum_{i=0}^{j-1}a^i$ with $a=\varepsilon_f$ and with $a=1-q_f$
gives $T_j=T_1\prod_f\sum_{i<j}\varepsilon_f^{\,i}$ and
$S_j(u)=S_1(u)\prod_f\sum_{i<j}(1-q_f)^i$, and the ratio is
\cref{eq:programs}. A factor of the product is at least one when
$(1-q_f)^i\ge\varepsilon_f^{\,i}$ for every $i$, that is, when
$q_f+\varepsilon_f\le1$. \hfill$\square$

\paragraph{Proof of \cref{thm:resample}.}
We use three facts about semialgebraic sets \citep[Chapter~2]{bochnak1998}. A
set defined by a first-order formula over the ordered field of the reals with
polynomial terms is semialgebraic (Tarski--Seidenberg). A semialgebraic set of
dimension $n$ in $\mathbb{R}^n$ has a nonempty interior. A semialgebraic set
of dimension smaller than $n$ is contained in a proper algebraic subset of
$\mathbb{R}^n$ and therefore has Lebesgue measure zero.

Fix a program. By \cref{ass:milp}, at $\theta$ it solves
$\min\{c(\theta)^{\top}(y,z):A(\theta)(y,z)\le b(\theta),\,y\in\mathbb{R}^p,\,z\in Z\}$,
where $Z$ is a finite set of integer vectors and the entries of $A$, $b$, and
$c$ are polynomials in $\theta$. A maximization is treated in the same way.
For a fixed $z\in Z$, the value
$w_z(\theta)=\inf\{c(\theta)^{\top}(y,z):A(\theta)(y,z)\le b(\theta)\}$ takes
values in $\mathbb{R}\cup\{-\infty,+\infty\}$. The sets on which it equals
$+\infty$ (infeasible) and $-\infty$ (unbounded), and its graph on the
remaining set, are defined by first-order formulas over the reals, so they are
semialgebraic. The output of the program is $\min_{z\in Z}w_z(\theta)$ together
with the corresponding status, and a minimum over the finite set $Z$ preserves
this property. The set $\Delta$ on which the outputs of two programs differ is
then defined by a first-order formula as well and is semialgebraic. If its
dimension is $n$, it has a nonempty interior and contains an open ball.
Otherwise it has Lebesgue measure zero. This proves (i).

For (ii), $Q$ has a density, so $Q(\Delta)=0$ when $\Delta$ has measure zero,
and no instance lies in $\Delta$ almost surely. When $\Delta$ contains a ball,
$\mu=Q(\Delta)>0$ because the density is positive almost everywhere, and the
$m$ independent instances all avoid $\Delta$ with probability $(1-\mu)^m$. For
(iii), programs that are correct on $\Theta$ return $V(x,\theta)$ at every
$\theta\in\Theta$ and therefore agree on every instance. The event $D$ does
not enter this argument. \hfill$\square$

The theorem is stated for exact agreement and for instances that perturb all
$n$ parameters. With the tolerance of \cref{eq:agree}, the same argument
applies to the set on which two outputs differ by more than the tolerance,
which is also semialgebraic. When only a subset of the parameters is
perturbed, the theorem applies on the corresponding coordinate subspace.

\paragraph{Proof of \cref{cor:verified}.}
Let $K$ be the event that the rule keeps the certificate. The argument of
\cref{prop:decomp} applies to the events $A\cap K$, $C\cap K$, and $W\cap K$,
so the error rate of the kept certificates is at most
$(\delta_k'+\rho_k')/(1+\rho_k')\le\delta_k'+\rho_k'$ with
$\delta_k'=\Pr(D\mid C\cap K)$ and $\rho_k'=\Pr(W\cap K)/\Pr(C\cap K)$. The
share $\delta_k'$ is given by \cref{eq:verify}. In addition,
$\Pr(W\cap K)\le\Pr(W)$ and $\Pr(C\cap K)=\kappa\Pr(C)$, so
$\rho_k'\le\rho_k/\kappa$. \hfill$\square$

\paragraph{Verification rates and optimal rules.}
Among the certificates with correct programs, a rule with rates $(d,\varphi)$
keeps a data error with probability $\delta_k(1-d)$ and a certificate without
a data error with probability $(1-\delta_k)(1-\varphi)$, which gives
\cref{eq:verify} and the kept fraction $1-\delta_k d-(1-\delta_k)\varphi$. Let
$\pi_1(s)$ and $\pi_0(s)$ be the probabilities of the flag pattern
$s\in\{0,1\}^r$ of $r$ verifiers for certificates with and without a data
error. A rule that withdraws on the set $S$ of patterns has
$d=\sum_{s\in S}\pi_1(s)$ and $\varphi=\sum_{s\in S}\pi_0(s)$. By the
Neyman--Pearson lemma \citep{lehmann2005}, a rule maximizes $d$ subject to
$\varphi\le\varphi_0$ if it withdraws the patterns in decreasing order of
$\pi_1(s)/\pi_0(s)$, with randomization on the last pattern. For two
verifiers, the rule ``B or C'' has $d\ge\max(d_{\mathrm{B}},d_{\mathrm{C}})$
and $\varphi\le\varphi_{\mathrm{B}}+\varphi_{\mathrm{C}}$, and the rule ``B
and C'' has $d\le\min(d_{\mathrm{B}},d_{\mathrm{C}})$ and
$\varphi\le\min(\varphi_{\mathrm{B}},\varphi_{\mathrm{C}})$.

\section{Configuration}
\label{app:config}

\paragraph{Models.}
Family~A is DeepSeek-V4-Flash (release of 31 July 2026) without its reasoning
mode, family~B is GPT-6.1 Sol, and family~C is Claude Sonnet~5.5. All three
are used at temperature zero through their providers' interfaces, and every
response of a family names the same model. A reply of family~A that contained
reasoning content despite this setting was requested again with identical
settings. This occurred only on IndustryOR and Mamo ComplexLP, and the
decision depended on the reply alone, not on the label. Family~A performs the extraction. Families A and C receive the candidate-generation
prompt of \cref{app:prompts} without rule~7 (direct strategy), and family~B
receives it with rule~7 (structured strategy). The modeling
interface named in rule~1 is Pyomo with the HiGHS solver for families A and B
and the Python interface of Gurobi for family~C. A program that raises an error or returns no value on the stated
instance is sent back once with the error message.

\Cref{alg:admitor} summarizes the certifier with data verification.

\begin{algorithm}[H]
\caption{\admitor{} certification of a problem text $x$}
\label{alg:admitor}
\begin{algorithmic}[1]
\STATE $\that\gets$ parameter set extracted from $x$ by family A
\FOR{each family $f\in\{\mathrm{A},\mathrm{B},\mathrm{C}\}$}
\STATE $p_f\gets$ program written by family $f$ from $x$ and the parameter names
\STATE $v_f\gets p_f(\that)$, with one revision if $p_f$ fails to execute
\ENDFOR
\IF{$v_{\mathrm{A}}$, $v_{\mathrm{B}}$, and $v_{\mathrm{C}}$ do not all agree by \cref{eq:agree}}
\STATE \textbf{return} abstain
\ENDIF
\FOR{each family $f\in\{\mathrm{B},\mathrm{C}\}$}
\STATE family $f$ compares $\that$ with $x$
\IF{it lists a mismatch or a value that is not stated}
\STATE \textbf{return} abstain
\ENDIF
\ENDFOR
\STATE \textbf{return} certificate $v_{\mathrm{A}}$
\end{algorithmic}
\end{algorithm}

\paragraph{Problem sets.}
The OptMATH-Train file has 201{,}071 records, numbered from zero. The first sample is
a uniform random sample of 300 records drawn with the pseudorandom generator
of the Python standard library and seed 42. Five of its problems were used while the prompts
were written and are excluded from every result, which leaves OptMATH-295.
OptMATH-200 is a uniform random sample of 200 of the 200{,}771 records outside
the first sample, drawn in the same way. Its record numbers were fixed and
published as a hash before any model was run on this set. The code of both
draws is part of the released code. NANO-CO has 245 problems of 49 types.
NANO-CO-150 is the subset whose labels were used to fit the threshold $\tau$ of
\cref{app:rules}, and NANO-CO-95 is the remainder. No problem text occurs in
both subsets.

\paragraph{Scoring.}
On the OptMATH sets, a certified value $v$ matches a label $\ell$ when the two
agree within relative tolerance $10^{-4}$ or when rounding $v$ to the number
of decimals printed in $\ell$ reproduces $\ell$. NANO-CO labels are integers or
have two decimals, and $v$ matches $\ell$ when rounding $v$ half up to two
decimals reproduces $\ell$ or when $|v-\ell|\le 0.005$. IndustryOR and Mamo ComplexLP are scored as the OptMATH sets. An answer of
the corrected releases that is written with a zero decimal, such as
$219816.0$, is read as an integer.

\paragraph{Replies that cannot be parsed.}
An extraction or verification reply that cannot be parsed is requested again
with identical settings, up to three attempts in total, and a reply that fails
three times counts as an abstention (extraction) or as no flag
(verification). With this rule, every extraction returned a parameter set. One
verification reply of family~C failed three times on each release of Mamo
ComplexLP.

\section{Prompts}
\label{app:prompts}

The three prompts are reproduced verbatim. Braces mark the fields that are
filled for each problem.

\paragraph{Extraction.}
\begin{Verbatim}[frame=single, framesep=2mm, fontsize=\small]
You are an optimization-structure extractor. From the problem below,
list the NUMERIC PARAMETERS (costs, prices, capacities, demands,
budgets, coefficients) together with their ACTUAL VALUES as given in
the problem data.

For each parameter also propose a perturbation domain for sensitivity
testing, anchored at the true values: keep it physically meaningful
but STRESS-TEST constraint activation (allow sign changes only where
economically plausible, e.g. a net profit that may turn negative;
widen enough that different constraints become binding). The domain
must remain feasible-plausible: do not propose ranges that contradict
the parameter's role (e.g. capacities near zero while demands stay
large).

Index-set sizes and cardinalities (number of employees, projects,
cars, ...) must NOT be perturbed: either OMIT them from params
entirely, or freeze them with a degenerate abs range (lo == hi ==
base). Changing a size without resizing every dependent array
produces structurally invalid instances.

Output ONLY a JSON object, no prose, exactly this schema:
{
  "params": {
    "<name>": {
      "meaning": "<short>",
      "base": <number or list, the TRUE values from the problem>,
      "perturb": {"mode": "rel", "r": <0..1>}
                 or {"mode": "abs", "lo": <num>, "hi": <num>,
                     "integer": <bool>}
      (in "abs" mode, lo and hi are ABSOLUTE values on the same scale
       as base, NOT offsets or deltas around it)
    }
  },
  "objective_sense": "max" or "min"
}

PROBLEM:
{question}
\end{Verbatim}

\paragraph{Candidate generation.}
\begin{Verbatim}[frame=single, framesep=2mm, fontsize=\small]
You are an expert in optimization modeling. Write a COMPLETE Python
script that defines a function solve(params) which builds and solves
the optimization problem below.

STRICT RULES:
1. Use {stack}.
2. Every numeric quantity listed in PARAMS must be read from the
   `params` dict (params["<name>"]); do NOT hardcode their values.
   Structural constants (set sizes, index structure) may be inline.
3. solve(params) must return a dict:
   {"objective": <float>, "status": "optimal"} on success;
   {"objective": None, "status": "<short reason>"} if infeasible/failed.
4. No printing, no file I/O, no network. Output ONLY one python code
   block.
5. Name the model object `model`; NEVER reuse `model` or `m` as a loop
   index or comprehension variable. Avoid Pyomo reserved attribute
   names for components (e.g. activate, deactivate, name, clone,
   index, display).
6. Do not wrap the whole body in a blanket try/except; let unexpected
   exceptions propagate (the harness captures them).
[structured strategy only]
7. Before the code, include as a Python comment a compact IR in JSON
   (sets, params, decision variables with types, constraints one line
   each, objective), then implement EXACTLY that IR.

PARAMS (name: meaning):
{param_desc}

PROBLEM:
{question}
\end{Verbatim}

\paragraph{Data verification.}
\begin{Verbatim}[frame=single, framesep=2mm, fontsize=\small]
You are checking a data extraction. Below are an optimization problem
statement and the parameter values that were extracted from it. For
every parameter, compare its value with the statement and report two
kinds of problems.
(1) mismatch: the statement gives a different number, a different
order of the entries, or assigns the value to a different entity.
(2) not_stated: the statement does not give the value. This includes
values for which the statement gives only examples, only a range, or
says that similar values apply, and values that were filled in to
complete a table.
Do not judge modeling choices. Do not report index sets, counts or
constants that follow directly from the statement. Return JSON only,
in the form {"mismatch": [{"parameter": "<name>", "extracted":
<value>, "stated": <value>, "evidence": "<short quote>"}],
"not_stated": [{"parameter": "<name>", "evidence": "<short quote or
'no value in the statement'>"}]}. If there is no problem, return
{"mismatch": [], "not_stated": []}.

STATEMENT:
{question}

PARAMETERS:
{one line per parameter: name | meaning | base value as JSON}
\end{Verbatim}

\section{Audit of the label misses}
\label{app:audit}

\paragraph{Procedure.}
For each audited problem, we read the problem text, the extracted parameter
set, the certified value, and the label. Class D requires that decisive
values of the parameter set do not occur in the text and that the text gives
them by example, by a range, or by a statement that similar values apply.
Class P requires that the text prints coefficients at reduced precision, that
the extracted values equal the printed ones, and that the certified value
lies within $0.1\%$ of the label. For the other classes, we re-solved the
problem from its text as a linear program, as a mixed-integer program, or by
enumeration. Class B means that the text is complete and the re-solve
reproduces the certified value. Class T means that the label is the optimum
of an instance that differs from the text in a stated rule or value. Class U
means that both the label and the certified value are optima under a reading
that the text supports. Class E means that the parameter set differs from the
text and the label is the optimum of the text. Class M means that the
parameter set equals the text and the label is the optimum of the text. On
OptMATH-200, the classes were assigned before the verification flags of these
problems were read.

\paragraph{Counts behind \cref{sec:audit}.}
The 63 label misses among the 587 three-family certificates consist of 30
data errors (27 of class D and three of class E) and 33 benchmark defects. If
the 524 certificates that match the label are taken as correct, the event $W$
does not occur, and the one-sided $95\%$ upper confidence bound on its rate is
$0.51\%$. The 110 certificates that only two families support carry 30 label
misses: three shared modeling errors, ten data errors, and 17 benchmark
defects. In one of the shared modeling errors, two programs weight a distance
by a demand where the text gives a cost per pair, and the third program
returns the label. From two to three families, the rate of shared modeling
errors falls from 3 in 697 certificates to none in 587, and the share of data
errors stays at $5.7\%$ and $4.9\%$. Counting the two-family certificates as
well, the audit identifies defects in 65 of the 495 sampled OptMATH-Train
problems ($13.1\%$) and in 21 of the 245 NANO-CO problems ($8.6\%$). Class D
is counted as a defect, because a text that omits decisive data does not
determine its label.

\paragraph{One case per class.}
The following cases quote the passages on which the classes rest.

\emph{Class D, record 130870 of OptMATH-Train.} The text describes six
producers and six contracts and gives the unit costs as follows: ``Producer 0
incurs a cost of 40 units for Contract 0, 46 units for Contract 1, and so on.
Producer 1 incurs a cost of 46 units for Contract 0, 40 units for Contract 1,
and so on. Similar cost structures apply to all other producers and
contracts.'' Four of the 36 costs are stated. The extraction contains 36
costs, the three programs return 370262, and the label is 357350. The
resampled rules certify the same value.

\emph{Class E, record 166455.} The text lists the demands by market:
``Market 0 requires 21 units of Product 0, 22 units of Product 1, 16 units of
Product 2.'' The extraction stores them as
$[[21,22,16],[25,16,23],[24,18,22]]$, one row per market. It stores the
profits with the product index before the market index, and the three
programs read the demand table in that order as well. They return 1806. The
optimum of the text is 1790, which is the label. Family~B reports the table in
its mismatch list, and families A and C raise no flag.

\emph{Class M, CFLP\_42 of NANO-CO.} The text states that ``the transportation
costs are derived from a pre-calculated distance matrix where the cost from
Hospital $i$ to Center $j$ is given by'' a matrix. The programs of families
B and C multiply this cost by the demand of the hospital and return 11855. The
program of family~A uses the cost per pair and returns the label, 1960.

\emph{Class P, record 198916.} The text gives 20 assignment costs with three
decimals, for example ``Assigning Peak 0 to Amino Acid 0 costs 0.537.'' Exactly
three assignments are required. The optimum of the printed costs is
$0.527+0.501+0.518=1.546$, which the three programs return. The label is
1.545455, the optimum of costs with more decimals.

\emph{Class B, record 89312.} The text is complete. The mixed-integer program
of the text has the optimal value 2396, which the three programs return, and
the label is 2482.

\emph{Class T, MC\_114 of NANO-CO.} The text lists capacities and unit costs
for the arcs and states that ``the cost of a cut is defined as the sum of the
arc\_unit\_costs of all directed arcs that cross from a node in set $S$ to a
node in set $T$.'' The minimum cut under this definition has cost 5. The label
is 13, the minimum cut with respect to the capacities.

\emph{Class U, mTSP\_161 of NANO-CO.} The text deploys two vans and states that
``all vehicles must depart from the Fulfillment Hub Node (Node 0), complete
service at their assigned Residential Delivery Addresses, and ultimately
return to the Fulfillment Hub Node.'' With two non-empty tours, the optimum is 358, which the programs
return. The label is 326, the length of a single tour that leaves the second
van unused.

\paragraph{OptMATH.}
\Cref{tab:audit-list} lists the problems of classes D and P by their record
number in OptMATH-Train, and \cref{tab:audit-om} gives the evidence for the
other classes.

\begin{table}[H]
\caption{OptMATH problems of classes D and P.}
\label{tab:audit-list}
\centering
\small
\begin{tabular}{lllp{0.58\linewidth}}
\toprule
Set & Agreement & Class & Records of OptMATH-Train \\
\midrule
OptMATH-295 & Three families & D & 25409, 29325, 32967, 37452, 40066, 68876, 69436, 94895, 110604, 118354, 122427, 130870, 138326, 140763, 154256, 163918, 168519, 178706, 195294, 195825 \\
 & Three families & P & 149, 850, 56160, 120285, 120435, 142853, 152969, 169392, 198916 \\
 & Two families only & D & 31720, 84975, 140568, 162141, 172949, 176518, 177392 \\
\midrule
OptMATH-200 & Three families & D & 26837, 66014, 89332, 99380, 168766, 188937, 189571 \\
 & Three families & P & 14681, 22825, 30081, 41525, 158046, 164155, 184248 \\
 & Two families only & D & 16675, 120613 \\
\bottomrule
\end{tabular}

\end{table}

\begin{table}[H]
\caption{OptMATH problems of classes B, T, U, and E. The first eleven rows are
three-family certificates, and the last three are certified by two families
only.}
\label{tab:audit-om}
\centering
\small
\begin{tabular}{rllrrp{0.46\linewidth}}
\toprule
Record & Set & Class & Certified & Label & Evidence \\
\midrule
64175 & OptMATH-295 & B & 75679 & 75312 & Lot sizing. The relaxation without capacities has optimum 75679, so the label is unattainable. \\
85008 & OptMATH-295 & B & 101235 & 100841 & The mixed-integer program of the text has optimum 101235. \\
52730 & OptMATH-295 & B & 122842 & 122814 & The mixed-integer program of the text has optimum 122842. \\
60043 & OptMATH-295 & T & 4702 & 3856 & Unit commitment. The label is the optimum when an initial fleet runs without start-up cost. The text states the same numbers as limits on start-ups in the first period. \\
175368 & OptMATH-295 & T & 4845 & 3898 & Same as record 60043. \\
166455 & OptMATH-295 & E & 1806 & 1790 & The demand table is stored by market and product and read by product and market. The optimum of the text is 1790. \\
59835 & OptMATH-200 & B & 277 & 276 & The linear program of the text has optimum 277. \\
77064 & OptMATH-200 & B & 10824 & 10969 & The linear program of the text has optimum 10824. \\
89312 & OptMATH-200 & B & 2396 & 2482 & The mixed-integer program of the text has optimum 2396. \\
144925 & OptMATH-200 & B & 1953 & 1987 & The mixed-integer program of the text has optimum 1953. \\
187975 & OptMATH-200 & B & 31277 & 31871 & The linear program of the text has optimum 31277. \\
\midrule
200145 & OptMATH-295 & B & 40911 & 41877.5 & The linear program of the text has maximum 40911, below the label. \\
26476 & OptMATH-295 & U & 17994 & 17514 & Binary assignment variables give 17994 and general integer variables give 17514. The text supports both. \\
28743 & OptMATH-295 & T & 2927 & 3140 & Unit commitment, as record 60043. \\
\bottomrule
\end{tabular}
\end{table}

\paragraph{NANO-CO.}
\Cref{tab:audit-nc} lists all 27 audited NANO-CO problems. Nine are
three-family certificates and 18 are certified by two families only.

\begin{table}[!t]
\caption{Audited NANO-CO problems. Families: number of agreeing families.}
\label{tab:audit-nc}
\centering
\small
\begin{tabular}{@{}llcrrlp{0.42\linewidth}@{}}
\toprule
Problem & Set & Families & Certified & Label & Class & Evidence \\
\midrule
CARP\_36 & NANO-CO-150 & 3 & 380 & 552 & T & The text allows direct travel between any two nodes, which gives 380. The labeled route uses the listed segments only. \\
CARP\_37 & NANO-CO-150 & 3 & 356 & 377 & T & Same as CARP\_36. \\
MCF\_121 & NANO-CO-150 & 3 & 65 & 127 & E & The parameter set places supplies and demands on other nodes than the text. \\
CPP\_51 & NANO-CO-150 & 2 & 1536 & 898 & M & The complete graph gives 788 and the listed segments give 898. The certified value is neither. \\
CPP\_52 & NANO-CO-150 & 2 & 629 & 719 & T & Direct travel gives 629 and the listed segments give 719. \\
CPP\_54 & NANO-CO-150 & 2 & 580 & 825 & T & Direct travel gives 580 and the listed segments give 825. \\
JSSP\_93 & NANO-CO-150 & 2 & 117 & 132 & E & The parameter set changes the machine and the duration of two operations. The text gives 132. \\
QAP\_185 & NANO-CO-150 & 2 & 2227 & 4454 & U & The label counts every pair of facilities twice. The text gives each flow once. \\
RPP\_188 & NANO-CO-150 & 2 & 493 & 655 & T & The text states that the network is fully connected. The labeled route uses the listed segments only. \\
RPP\_189 & NANO-CO-150 & 2 & 524 & 798 & T & Same as RPP\_188. \\
VRP\_244 & NANO-CO-150 & 2 & 276 & 261 & U & The label leaves one vehicle unused. The text states that both vehicles start from the depot. \\
VRP\_245 & NANO-CO-150 & 2 & 238 & 233 & U & Same as VRP\_244. \\
mTSP\_162 & NANO-CO-150 & 2 & 354 & 318 & U & Same as VRP\_244. \\
mTSP\_164 & NANO-CO-150 & 2 & 272 & 235 & U & Same as VRP\_244. \\
mTSP\_165 & NANO-CO-150 & 2 & 353 & 304 & U & Same as VRP\_244. \\
\midrule
MCC\_120 & NANO-CO-95 & 3 & 86 & 83 & T & The text puts a lower bound on arc $4\to5$ and the stored instance on arc $5\to4$. The text gives 86. \\
MC\_114 & NANO-CO-95 & 3 & 5 & 13 & T & The text defines the cost of a cut by unit costs, which gives 5. The label is the capacity of the cut. \\
RPP\_187 & NANO-CO-95 & 3 & 554 & 712 & T & Direct travel gives 554 and the listed segments give 712. \\
SM\_197 & NANO-CO-95 & 3 & 13 & 12 & E & The parameter set changes one row of the preference table. The text gives 12. \\
mTSP\_161 & NANO-CO-95 & 3 & 358 & 326 & U & One tour gives 326 and two non-empty tours give 358. \\
mTSP\_163 & NANO-CO-95 & 3 & 324 & 305 & U & One tour gives 305 and two non-empty tours give 324. \\
CARP\_38 & NANO-CO-95 & 2 & 328 & 382 & T & Direct travel gives 328. \\
CFLP\_42 & NANO-CO-95 & 2 & 11855 & 1960 & M & Two programs weight the distance by the demand. The text gives a cost per pair, and the third program returns 1960. \\
CPP\_55 & NANO-CO-95 & 2 & 418 & 501 & T & Direct travel gives 418 and the listed segments give 501. \\
RPP\_186 & NANO-CO-95 & 2 & 515 & 665 & T & Same as RPP\_188. \\
TSPTW\_224 & NANO-CO-95 & 2 & 367.49 & 368 & M & Two programs use unrounded coordinate distances instead of the stated travel times. \\
VRP\_242 & NANO-CO-95 & 2 & 362 & 328 & U & Same as VRP\_244. \\
\bottomrule
\end{tabular}
\end{table}

\section{Data verification}
\label{app:verify}

\Cref{tab:check-family} gives the flags of each family on the three-family
certificates of each set. Under the retry rule of \cref{app:config}, every
verification reply on these sets could be parsed. On OptMATH-295, family~B was
queried on 273 problems, which include every three-family certificate.
\Cref{tab:overflag} classifies the 40 label-matching certificates that the
rule withdraws.

\begin{table}[H]
\caption{Number of three-family certificates on which a flag is raised, by
whether the certificate matches the label.}
\label{tab:check-family}
\centering
\small
\begin{tabular}{lrrrrrrrr}
\toprule
 & \multicolumn{2}{c}{OptMATH-295} & \multicolumn{2}{c}{OptMATH-200} & \multicolumn{2}{c}{NANO-CO-150} & \multicolumn{2}{c}{NANO-CO-95} \\
\cmidrule(lr){2-3}\cmidrule(lr){4-5}\cmidrule(lr){6-7}\cmidrule(lr){8-9}
Flag raised by & Match & Miss & Match & Miss & Match & Miss & Match & Miss \\
\midrule
Family A (the extractor's family) & 14 & 19 & 12 & 7 & 5 & 1 & 0 & 0 \\
Family B & 14 & 21 & 9 & 7 & 4 & 1 & 0 & 1 \\
Family C & 13 & 20 & 9 & 8 & 4 & 1 & 1 & 1 \\
B or C & 21 & 21 & 14 & 8 & 4 & 1 & 1 & 1 \\
B and C & 6 & 20 & 4 & 7 & 4 & 1 & 0 & 1 \\
At least two of A, B, C & 10 & 20 & 7 & 7 & 4 & 1 & 0 & 1 \\
\midrule
Certificates & 208 & 35 & 149 & 19 & 103 & 3 & 64 & 6 \\
\bottomrule
\end{tabular}

\end{table}

\begin{table}[H]
\caption{Review of the label-matching three-family certificates that are
withdrawn when family B or C raises a flag.}
\label{tab:overflag}
\centering
\small
\begin{tabular}{lrr}
\toprule
Finding of the review & Development sets & OptMATH-200 \\
\midrule
Flag correct: the extraction differs from the text and the optimum is unchanged & 5 & 0 \\
Flag correct: data are absent from the text & 1 & 4 \\
Over-flag: a complete list that the text introduces with ``for example'' & 7 & 3 \\
Over-flag: a constant that follows from the text & 5 & 3 \\
Over-flag: the diagonal of a matrix & 5 & 2 \\
Over-flag: the reply is unrelated to the data & 2 & 2 \\
Over-flag: a stated value reported as different & 1 & 0 \\
\midrule
Withdrawn label-matching certificates & 26 & 14 \\
\bottomrule
\end{tabular}
\end{table}

\paragraph{Flags by type.}
\Cref{tab:flag-types} separates the two lists of the reply. The not-stated
list accounts for all 26 certificates of class D, and the mismatch list for
the three of class E. Most flags on label-matching certificates are also in
the not-stated list, which is where the over-flags of the first three kinds of
\cref{tab:overflag} arise.

\begin{table}[H]
\caption{Three-family certificates on which a flag is raised, by list of the
reply and by label match, over the four sets.}
\label{tab:flag-types}
\centering
\small
\begin{tabular}{lcccccc}
\toprule
 & \multicolumn{2}{c}{Mismatch} & \multicolumn{2}{c}{Not stated} & \multicolumn{2}{c}{Either} \\
\cmidrule(lr){2-3}\cmidrule(lr){4-5}\cmidrule(lr){6-7}
Flag raised by & Match & Miss & Match & Miss & Match & Miss \\
\midrule
Family A & 5 & 4 & 27 & 26 & 31 & 27 \\
Family B & 5 & 4 & 23 & 27 & 27 & 30 \\
Family C & 9 & 4 & 19 & 27 & 27 & 30 \\
B or C & 9 & 6 & 33 & 27 & 40 & 31 \\
B and C & 5 & 2 & 9 & 27 & 14 & 29 \\
\midrule
Certificates & 524 & 63 & 524 & 63 & 524 & 63 \\
\bottomrule
\end{tabular}

\end{table}

\begin{table}[H]
\caption{Joint flags of families B and C on three-family certificates, and the
likelihood ratio of each pattern between the two rows.}
\label{tab:joint}
\centering
\small
\setlength{\tabcolsep}{4pt}
\begin{tabular}{lcccc}
\toprule
Certificates & Both & B only & C only & Neither \\
\midrule
Data error (D, E) & 29 & 1 & 0 & 0 \\
Matching the label & 14 & 13 & 13 & 484 \\
\midrule
Likelihood ratio & 36.2 & 1.3 & 0.0 & 0.0 \\
\bottomrule
\end{tabular}

\end{table}

\paragraph{Counts behind \cref{sec:verify}.}
On OptMATH-200, the 19 label misses were classified before the flags were
read: seven of class D and twelve of classes P and B. The rule withdraws the seven
and one of the twelve, a certificate of class B on which family~C reports a
stated value as different. The not-stated list reports all 26 certificates of
class D, and the mismatch list reports the three of class E. After
verification, 516 certificates remain, with 32 label misses, all benchmark
defects. Label precision rises by $7.4$ points on OptMATH-295 and by $3.8$
points on OptMATH-200. The rule withdraws 40 of the 524 label-matching
certificates. Ten of these flags are correct: in five, the extraction differs
from the text without changing the optimum, and in five, the text omits data
and the certified value still matches the label. The bound of
\cref{cor:verified} uses $\hat\delta_3=0.051$, $\hat\varphi=0.076$, $d\ge0.90$,
and $\rho_3\le0.0051$. Both families flag 29 of the 30 data errors and 14 of
the 524 label-matching certificates, family~B alone flags one and 13, and
family~C alone flags none and 13. The stricter rule saves 26 withdrawals.
Family~A flags 27 of the 30 data errors, against 30 for family~B and 29 for
family~C. Among the 110 certificates that only two families support,
verification withdraws eight of the ten data errors and none of the three
shared modeling errors.

\paragraph{Dependence between the two verifiers.}
\Cref{tab:joint} gives the joint flags of families B and C and the likelihood
ratio of each pattern between the two groups of certificates, which orders the
rules as described in \cref{sec:method-verify}. On the 30
certificates with a data error, the two families agree in 29 cases. On the 524
label-matching certificates, both raise a flag on 14. If the two flags were
independent with their marginal rates of 27 and 27 in 524, 1.4 joint flags
would be expected. The flags on label-matching certificates therefore have
common causes in the text, such as a list introduced by ``for example'', and a
second verifier removes fewer over-flags than independence would suggest.

\paragraph{An over-flag.}
Record 88236 of OptMATH-Train states bounds on the frequency of one transit
line and adds: ``The same frequency bounds apply to all other lines ($L_1$ to $L_8$).'' The
extraction assigns these bounds to every line, which is what the text says.
All three families report the bounds as not stated, and families A and C quote
this sentence as evidence. The certificate matches the label and is withdrawn.

\section{Resampling rules and a second program per family}
\label{app:rules}

\paragraph{Rules.}
The extraction prompt asks for a perturbation domain for each parameter. The
primary sampler perturbs each parameter alone, twice, with all other
parameters at their stated values. It keeps a parameter when the programs
that return a value on the stated instance still return one, and it then
draws $m=5$ instances that perturb the kept parameters jointly. The
independent sampler draws every parameter from its domain in each instance. An
instance is compared when at least two programs return a value on it. A
resampled rule finds the largest group of programs that agree in the sense of
\cref{eq:agree} on the stated instance and on every compared instance, and it
requires at least three compared instances. The two-family and three-family
variants require that the group spans two or three families. The threshold
variant computes the score $s=10F+|Q|+n/10$, where $F$ is the number of
families in the group, $|Q|$ the number of programs, and $n$ the number of
compared instances, and certifies when $s\ge\tau$. We set $\tau=33.4$, the
smallest threshold for which the $95\%$ Clopper--Pearson upper bound on the
share of label misses among the selected NANO-CO-150 certificates is below
$0.05$. With one program per family, this threshold selects three-family
groups with at least four compared instances.

\paragraph{Results.}
\Cref{tab:rules-full} reports every rule on every set. The precision of the
threshold rule on NANO-CO-150 is computed on the labels used to fit $\tau$.
\Cref{tab:k6-full} compares one and two programs per family for every rule. With two programs per
family, a rule requires a group of agreeing programs that spans the stated
number of families.

\begin{table}[H]
\caption{Coverage and label precision ($\pm$ binomial standard error) of all
rules. \emph{Resampled}: primary sampler. \emph{Verified}: the certificate is
withdrawn when the named families raise a flag.}
\label{tab:rules-full}
\centering
\small
\setlength{\tabcolsep}{3.5pt}
\begin{tabular}{lcccccccc}
\toprule
 & \multicolumn{2}{c}{OptMATH-295} & \multicolumn{2}{c}{OptMATH-200} & \multicolumn{2}{c}{NANO-CO-150} & \multicolumn{2}{c}{NANO-CO-95} \\
\cmidrule(lr){2-3}\cmidrule(lr){4-5}\cmidrule(lr){6-7}\cmidrule(lr){8-9}
Rule & Cov. & Precision & Cov. & Precision & Cov. & Precision & Cov. & Precision \\
\midrule
Family A alone & 0.898 & $0.807\pm0.024$ & 0.890 & $0.854\pm0.026$ & 0.887 & $0.820\pm0.033$ & 0.842 & $0.863\pm0.038$ \\
Family B alone & 0.922 & $0.838\pm0.022$ & 0.945 & $0.878\pm0.024$ & 0.993 & $0.893\pm0.025$ & 0.990 & $0.872\pm0.034$ \\
Family C alone & 0.966 & $0.796\pm0.024$ & 0.955 & $0.853\pm0.026$ & 0.993 & $0.886\pm0.026$ & 0.990 & $0.840\pm0.038$ \\
\midrule
Two families agree & 0.925 & $0.835\pm0.022$ & 0.920 & $0.886\pm0.023$ & 0.980 & $0.898\pm0.025$ & 0.979 & $0.871\pm0.035$ \\
Three families agree & 0.824 & $0.856\pm0.022$ & 0.840 & $0.887\pm0.024$ & 0.707 & $0.972\pm0.016$ & 0.737 & $0.914\pm0.034$ \\
\midrule
Resampled, two families & 0.885 & $0.839\pm0.023$ & 0.890 & $0.888\pm0.024$ & 0.913 & $0.905\pm0.025$ & 0.842 & $0.887\pm0.035$ \\
Resampled, three families & 0.769 & $0.868\pm0.022$ & 0.805 & $0.894\pm0.024$ & 0.633 & $0.990\pm0.011$ & 0.632 & $0.917\pm0.036$ \\
Resampled, threshold & 0.759 & $0.866\pm0.023$ & 0.795 & $0.893\pm0.025$ & 0.633 & $0.990\pm0.011$ & 0.621 & $0.915\pm0.036$ \\
\midrule
Independent draws, two families & 0.719 & $0.821\pm0.026$ & 0.735 & $0.878\pm0.027$ & 0.767 & $0.904\pm0.027$ & 0.821 & $0.885\pm0.036$ \\
Independent draws, three families & 0.607 & $0.855\pm0.026$ & 0.655 & $0.885\pm0.028$ & 0.553 & $1.000\pm0.000$ & 0.590 & $0.911\pm0.038$ \\
Independent draws, threshold & 0.590 & $0.851\pm0.027$ & 0.630 & $0.881\pm0.029$ & 0.547 & $1.000\pm0.000$ & 0.579 & $0.909\pm0.039$ \\
\midrule
Three families, verified (B or C) & 0.681 & $0.930\pm0.018$ & 0.730 & $0.925\pm0.022$ & 0.673 & $0.980\pm0.014$ & 0.716 & $0.926\pm0.032$ \\
Three families, verified (B and C) & 0.736 & $0.931\pm0.017$ & 0.785 & $0.924\pm0.021$ & 0.673 & $0.980\pm0.014$ & 0.726 & $0.927\pm0.031$ \\
Two families, verified (B or C) & 0.756 & $0.915\pm0.019$ & 0.780 & $0.929\pm0.021$ & 0.907 & $0.919\pm0.023$ & 0.937 & $0.876\pm0.035$ \\
Two families, verified (B and C) & 0.820 & $0.917\pm0.018$ & 0.860 & $0.924\pm0.020$ & 0.940 & $0.908\pm0.024$ & 0.958 & $0.879\pm0.034$ \\
\bottomrule
\end{tabular}

\end{table}

\begin{table}[H]
\caption{One and two programs per family.}
\label{tab:k6-full}
\centering
\small
\begin{tabular}{llcccc}
\toprule
 & & \multicolumn{2}{c}{OptMATH-295} & \multicolumn{2}{c}{NANO-CO-95} \\
\cmidrule(lr){3-4}\cmidrule(lr){5-6}
Rule & Candidates & Coverage & Precision & Coverage & Precision \\
\midrule
Two families agree & 1 per family & 0.925 & $0.835\pm0.022$ & 0.979 & $0.871\pm0.035$ \\
Two families agree & 2 per family & 0.912 & $0.825\pm0.023$ & 0.947 & $0.878\pm0.035$ \\
Three families agree & 1 per family & 0.824 & $0.856\pm0.022$ & 0.737 & $0.914\pm0.034$ \\
Three families agree & 2 per family & 0.831 & $0.849\pm0.023$ & 0.821 & $0.910\pm0.032$ \\
Resampled, two families & 1 per family & 0.885 & $0.839\pm0.023$ & 0.842 & $0.887\pm0.035$ \\
Resampled, two families & 2 per family & 0.881 & $0.823\pm0.024$ & 0.863 & $0.854\pm0.039$ \\
\bottomrule
\end{tabular}

\end{table}

\begin{table}[H]
\caption{Fit of the threshold $\tau$ on NANO-CO-150 with the primary sampler.
\emph{Upper bound}: one-sided $95\%$ Clopper--Pearson bound on the share of
label misses among the selected certificates.}
\label{tab:tau}
\centering
\small
\begin{tabular}{ccccc}
\toprule
$\tau$ & Selected & Label misses & Share & Upper bound \\
\midrule
22.3 & 137 & 13 & 0.0949 & 0.1466 \\
22.4 & 135 & 13 & 0.0963 & 0.1487 \\
22.6 & 134 & 12 & 0.0896 & 0.1411 \\
33.4 & 95 & 1 & 0.0105 & 0.0490 \\
33.5 & 94 & 1 & 0.0106 & 0.0495 \\
33.6 & 91 & 1 & 0.0110 & 0.0511 \\
\bottomrule
\end{tabular}
\end{table}

\paragraph{Data errors under resampling.}
Of the 30 three-family certificates with a data error, the threshold rule with
the primary sampler keeps 25: 18 of 21 on OptMATH-295, six of seven on
OptMATH-200, the one on NANO-CO-150, and not the one on NANO-CO-95. Over all
three-family certificates, it keeps 537 of 587. Resampling keeps most data
errors, in line with part (iii) of \cref{thm:resample}.

\paragraph{Dependence between the families.}
\Cref{tab:joint-match} compares the share of problems on which all three
families return a value that matches the label with the product of the three
single-family shares. The observed share exceeds the product on every set, by
a factor between $1.2$ and $1.65$. Under \cref{ass:ci}, this excess is the
covariance of the probabilities $1-\varepsilon_f$ across problems. Problems
with a defective text or label, on which no family matches, contribute to it.

\begin{table}[H]
\caption{Share of problems on which each family returns a value that matches
the label, the product of the three shares, and the observed share on which
all three match.}
\label{tab:joint-match}
\centering
\small
\begin{tabular}{@{}lccccc@{}}
\toprule
 & \multicolumn{3}{c}{Family} & \multicolumn{2}{c}{All three} \\
\cmidrule(lr){2-4}\cmidrule(l){5-6}
Set & A & B & C & Product & Observed \\
\midrule
OptMATH-295 & 0.725 & 0.773 & 0.769 & 0.431 & 0.705 \\
OptMATH-200 & 0.760 & 0.830 & 0.815 & 0.514 & 0.745 \\
NANO-CO-150 & 0.727 & 0.887 & 0.880 & 0.567 & 0.687 \\
NANO-CO-95 & 0.726 & 0.863 & 0.832 & 0.521 & 0.674 \\
\bottomrule
\end{tabular}

\end{table}

\paragraph{Single families.}
\Cref{tab:single-vs} gives, for each set, the problems without three-family
agreement and the values that each family returns on them. These counts
underlie \cref{fig:main}(a).

\begin{table}[H]
\caption{Problems without three-family agreement: values returned by each
family and the number that miss the label.}
\label{tab:single-vs}
\centering
\small
\begin{tabular}{lccccccc}
\toprule
 & Problems without & \multicolumn{2}{c}{Family A} & \multicolumn{2}{c}{Family B} & \multicolumn{2}{c}{Family C} \\
\cmidrule(lr){3-4}\cmidrule(lr){5-6}\cmidrule(lr){7-8}
Set & agreement & Values & Misses & Values & Misses & Values & Misses \\
\midrule
OptMATH-295 & 52 & 22 & 16 & 29 & 9 & 42 & 23 \\
OptMATH-200 & 32 & 10 & 7 & 21 & 4 & 23 & 9 \\
NANO-CO-150 & 44 & 27 & 21 & 43 & 13 & 43 & 14 \\
NANO-CO-95 & 25 & 10 & 5 & 24 & 6 & 24 & 9 \\
\bottomrule
\end{tabular}

\end{table}

\section{Human-written texts}
\label{app:human}

\paragraph{Releases.}
The original releases are IndustryOR (100 problems) and the complex linear
programming part of Mamo (211 problems), with their published labels. The
expert audit of \citet{orsurvey2025} publishes cleaned files that keep 42 of
the IndustryOR problems and 111 of the Mamo ComplexLP problems, one of them
with a changed label. We match problems to these files by their text after
lowercasing and collapsing whitespace. The corrected releases of
\citet{optimind2025} contain 99 and 210 problems with revised texts and
answers (repository \texttt{microsoft/OptiGuide}, commit \texttt{3da334a}).
The corrected IndustryOR release also replaces some problems, and we run both
corrected releases as published. Models, prompts, sampler, threshold, and
verification rule are those of the other sets. \Cref{tab:human-rules} gives
every rule on both releases.

\begin{table}[H]
\caption{Certification rules on the original and the corrected releases of
IndustryOR and Mamo ComplexLP.}
\label{tab:human-rules}
\centering
\footnotesize
\setlength{\tabcolsep}{3pt}
\begin{tabular}{lcccccccc}
\toprule
 & \multicolumn{4}{c}{Original releases} & \multicolumn{4}{c}{Corrected releases} \\
\cmidrule(lr){2-5}\cmidrule(lr){6-9}
 & \multicolumn{2}{c}{IndustryOR} & \multicolumn{2}{c}{Mamo ComplexLP} & \multicolumn{2}{c}{IndustryOR} & \multicolumn{2}{c}{Mamo ComplexLP} \\
\cmidrule(lr){2-3}\cmidrule(lr){4-5}\cmidrule(lr){6-7}\cmidrule(lr){8-9}
Rule & Cov. & Precision & Cov. & Precision & Cov. & Precision & Cov. & Precision \\
\midrule
Family A alone & 0.790 & $0.405\pm0.055$ & 0.900 & $0.516\pm0.036$ & 0.869 & $0.744\pm0.047$ & 0.929 & $0.846\pm0.026$ \\
Family B alone & 0.860 & $0.407\pm0.053$ & 0.962 & $0.596\pm0.034$ & 0.990 & $0.847\pm0.036$ & 0.986 & $0.942\pm0.016$ \\
Family C alone & 0.870 & $0.437\pm0.053$ & 0.967 & $0.534\pm0.035$ & 0.970 & $0.792\pm0.041$ & 0.981 & $0.874\pm0.023$ \\
\midrule
Two families agree & 0.730 & $0.520\pm0.059$ & 0.929 & $0.582\pm0.035$ & 0.889 & $0.886\pm0.034$ & 0.967 & $0.926\pm0.018$ \\
Three families agree & 0.460 & $0.587\pm0.073$ & 0.763 & $0.547\pm0.039$ & 0.596 & $0.966\pm0.024$ & 0.786 & $0.939\pm0.019$ \\
Three families, resampled & 0.290 & $0.586\pm0.091$ & 0.687 & $0.524\pm0.042$ & 0.434 & $0.954\pm0.032$ & 0.705 & $0.932\pm0.021$ \\
\midrule
Two families, verified & 0.620 & $0.532\pm0.063$ & 0.730 & $0.591\pm0.040$ & 0.798 & $0.911\pm0.032$ & 0.767 & $0.944\pm0.018$ \\
Three families, verified & 0.400 & $0.625\pm0.076$ & 0.611 & $0.566\pm0.044$ & 0.525 & $0.981\pm0.019$ & 0.676 & $0.944\pm0.019$ \\
Three, verified (B and C) & 0.430 & $0.628\pm0.074$ & 0.692 & $0.575\pm0.041$ & 0.576 & $0.983\pm0.017$ & 0.748 & $0.943\pm0.019$ \\
\bottomrule
\end{tabular}

\end{table}

\paragraph{Audit.}
We audited every three-family certificate that misses the label, on both
releases, with the procedure of \cref{app:audit}. \Cref{tab:human-audit} gives
the classes. The class \emph{integrality not stated} contains certificates of
a program with continuous quantities whose label is the optimum of the same
program with integer quantities, on texts that do not state whether the
quantities are integer. Under the classes of \cref{sec:audit}, these
certificates belong to class U on the published text. Under the reading that
the corrected release states, they are instances of the event $W$.

\begin{table}[H]
\caption{Three-family certificates that miss the label on IndustryOR and Mamo
ComplexLP, by class. On the original releases, the columns separate the
problems that the expert audit keeps from those it removes.}
\label{tab:human-audit}
\centering
\small
\begin{tabular}{lccc}
\toprule
 & \multicolumn{2}{c}{Original releases} & Corrected \\
\cmidrule(lr){2-3}
Class & Kept by audit & Removed by audit & releases \\
\midrule
Data error (D, E) & 1 & 8 & 2 \\
Shared modeling error (M) & 0 & 0 & 0 \\
Integrality not stated & 28 & 1 & 4 \\
Label not the optimum of the text (B) & 6 & 37 & 5 \\
Other text defects (T, U) & 1 & 10 & 1 \\
\midrule
Label misses & 36 & 56 & 12 \\
Withdrawn by verification & 1 & 20 & 3 \\
Three-family certificates & 128 & 79 & 224 \\
\bottomrule
\end{tabular}

\end{table}

On the original releases, 56 of the 92 label misses fall on problems that the
expert audit removes, and one more on the problem whose label it changes,
although these problems carry 80 of the 207 certificates. The eight data errors
on removed problems come from texts that give values as a range, end a list
with ``and so on'', state the data symbolically, or describe a graph by a
picture that is not part of the text. Verification withdraws all eight and the
one data error on a kept problem, whose text prints a table with shifted
columns. The kept problems also carry six labels that are not the optimum of
the stated problem. On four traveling-salesman problems with four cities, the
label lies below the cost of every tour. On a production problem, the label
is the cost of producing every product without the substitution that the text
allows, and on a planning problem, no reading of the text that we tried
reproduces the label. IndustryOR contains one problem twice, with the labels
$8{,}505$ and $10{,}755$, and the expert audit keeps both.

On the corrected releases, the two data errors are an extraction that reads
``the same total capacity'' as half of it and a text that gives the capacities
of nine centers only as a range. A third withdrawn certificate concerns a text
that states its protein requirement twice. For five of the remaining nine
misses, our re-solve of the stated problem reproduces the certified value: a
traveling-salesman answer below the cost of every tour, three facility-location
answers that differ from the optimum of the stated data, one of them by a
missing digit, and an area that differs in the fourth significant digit. The
other four are certificates of continuous programs on texts that do not state
integrality.

\paragraph{Texts that state integrality.}
The corrected release of Mamo ComplexLP contains 44 diet problems. On the 37
whose text includes the sentence on integral portions, three-family agreement
certifies 36, and 35 match the answer. The one miss is the text that states its
protein requirement twice, and its three programs use integer quantities. On
the seven diet problems without the sentence, it certifies all seven, and four match.
The 26 diet problems of the original release whose certificates miss the label
because of continuous quantities are all among the 37. On each of them, the
three programs of the corrected release return the optimum with integer
portions, and the certificate matches the answer.

\paragraph{Verification.}
On the original releases, verification withdraws 21 of the 92 label misses and
17 of the 115 label-matching certificates. On the corrected releases, it
withdraws three of the twelve misses and 27 of the 212 matching certificates.

\section{Certificates as the admission rule of an experience library}
\label{app:downstream}

\paragraph{Design.}
We use OptSkills \citep{optskills2026} as the host agent, with
DeepSeek-V4-Flash as its model. The host solves each problem of
OptMATH-295 with several solution attempts and distills the admitted attempts
into a skill library, which it then consults on new problems. Its admission
rule compares the value of an attempt with the reference label. We replace
this rule by four label-free rules. The first three admit an attempt when its
value agrees with the three-family certificate, with the two-family
certificate, or with the most frequent value among the attempts of the problem
(majority vote). The fourth admits every attempt that runs without error
(execution success). We evaluate each
library on 1{,}082 problems from four benchmarks
\citep{industryor2024,mamo2024,optmath2025,optibench2024}. We
also run the host without a learned library and with the library that the
host's authors released. Accuracy is the share of problems
whose answer matches the label under the scoring rule of \cref{app:config}.

\paragraph{Admission.}
\Cref{tab:arms} reports what each rule admits. The host produces 862 attempts
that run without error on OptMATH-295, of which 669 match the label. Precision
is the share of admitted attempts that match the label, and recall is the
share of the 669 matching attempts that are admitted. The certificate rules
admit a purer set than majority vote and execution success, at a lower recall.

\begin{table}[H]
\caption{Attempts admitted into the library by each rule (precision and recall
$\pm$ binomial standard error).}
\label{tab:arms}
\centering
\small
\begin{tabular}{lcccc}
\toprule
Admission rule & Admitted & Wrong & Precision & Recall \\
\midrule
Execution success & 862 & 193 & $0.776\pm0.014$ & $1.000\pm0.000$ \\
Majority vote & 758 & 94 & $0.876\pm0.012$ & $0.993\pm0.003$ \\
Two families agree & 712 & 67 & $0.906\pm0.011$ & $0.964\pm0.007$ \\
Three families agree & 644 & 53 & $0.918\pm0.011$ & $0.883\pm0.012$ \\
Three families, resampled & 601 & 51 & $0.915\pm0.011$ & $0.822\pm0.015$ \\
Reference labels & 669 & 0 & $1.000\pm0.000$ & $1.000\pm0.000$ \\
\bottomrule
\end{tabular}

\end{table}

\paragraph{Results.}
\Cref{tab:downstream} reports the accuracy of each configuration. The five
libraries differ in the share of wrong attempts that they admit, from none to
$22.4\%$, and their accuracies differ by at most $1.7$ points. In paired exact
McNemar tests on the 1{,}082 problems, the three-family library and the
two-family library are more accurate than the library of execution success
($p=0.002$ and $p=0.004$), and no other pair of libraries is separated at the
$1\%$ level. The host without a learned library is more accurate than every
library, by $3.5$ to $5.2$ points. In paired tests (\cref{tab:paired}), it
differs from the three-family library at $p=0.006$, from the two-family library
at $p=0.004$, and from the other three libraries at $p<0.001$. Most of the
difference comes from OptiBench. The released library reaches $63.03\%$ in its
one pass. In paired tests on first passes, it does not differ from the host
without a learned library ($p=0.66$). It is more accurate than the libraries
of majority vote, execution success, and reference labels ($p$ between
$0.0001$ and $0.012$), and it is not separated from the three-family and
two-family libraries ($p=0.10$ and $p=0.08$).

\begin{table}[H]
\caption{Accuracy in percent on the evaluation benchmarks (mean $\pm$ standard
error over three passes; a single value for the three libraries with one
pass). \emph{Admitted}:
number of attempts admitted into the library. \emph{Wrong}: share of admitted
attempts whose value misses the label, in percent. Mamo.C is Mamo.Complex,
OptMATH-B is OptMATH-Bench, and \emph{All} is the average over the 1{,}082
problems.}
\label{tab:downstream}
\centering
\footnotesize
\setlength{\tabcolsep}{3.2pt}
\begin{tabular}{@{}lrrccccc@{}}
\toprule
Admission rule & Admitted & Wrong & IndustryOR & Mamo.C & OptMATH-B & OptiBench & All \\
\midrule
No learned library & -- & -- & $34.67\pm0.33$ & $56.40\pm0.99$ & $57.63\pm0.72$ & $71.85\pm0.53$ & $63.22\pm0.18$ \\
\midrule
Three families agree & 644 & 8.2 & $33.67\pm0.88$ & $53.71\pm1.56$ & $55.42\pm2.43$ & $67.27\pm0.76$ & $59.70\pm0.59$ \\
Two families agree & 712 & 9.4 & $36.33\pm1.45$ & $53.87\pm0.32$ & $56.22\pm1.00$ & $66.45\pm0.58$ & $59.64\pm0.49$ \\
Majority vote & 758 & 12.4 & $35.67\pm0.88$ & $52.45\pm0.84$ & $57.63\pm0.40$ & $65.18\pm0.72$ & $58.81\pm0.66$ \\
Execution success & 862 & 22.4 & $33.00$ & $50.71$ & $57.83$ & $64.79$ & $58.04$ \\
Reference labels & 669 & 0.0 & $34.00$ & $50.24$ & $55.42$ & $67.77$ & $59.33$ \\
\midrule
Host's released library & -- & -- & $27.00$ & $56.87$ & $60.84$ & $71.74$ & $63.03$ \\
\bottomrule
\end{tabular}

\end{table}

\begin{table}[H]
\caption{Paired comparison of each library with the host without a learned
library on the 1{,}082 problems: number of problems that only one of the two
configurations solves, and the exact McNemar $p$-value.}
\label{tab:paired}
\centering
\small
\begin{tabular}{llccc}
\toprule
Library & Answers compared & Library only & No library only & $p$ \\
\midrule
Three families agree & Majority of three passes & 31 & 58 & 0.006 \\
Two families agree & Majority of three passes & 30 & 58 & 0.004 \\
Majority vote & Majority of three passes & 31 & 68 & $<0.001$ \\
Execution success & First pass & 45 & 105 & $<0.001$ \\
Reference labels & First pass & 44 & 90 & $<0.001$ \\
Host's released library & First pass & 63 & 69 & 0.664 \\
\bottomrule
\end{tabular}

\end{table}

\paragraph{Reading.}
Under our library build, the purity of the admitted set does not change the
accuracy of the host, and every library that we built lowers it. The build
explains the second part. Our libraries contain 157 to 209 skills, and about $90\%$ of the
skills are distilled from a single problem (\cref{tab:libdiag}). The reason is that the build clusters the
problems
with a minimum cluster size of one and runs no learning phase after
clustering. The ablation of the host's authors reports the same pattern for
skills distilled per problem without clustering. Accuracy on OptiBench falls
below that of the agent without skills ($69.75$ against $73.22$), and the
authors attribute the result to overlapping and instance-specific skills
\citep{optskills2026}. The released library has 56 skills and, according
to the host's documentation, was built with clustering and a learning phase.
The host's repository contains no record of that build. Its accuracy does not
differ from that of the host without a learned library. The host's authors report
higher accuracy with their library than without it on all benchmarks of their evaluation
with an earlier model \citep{optskills2026}. The results therefore concern
this host model and these library builds and not experience libraries in
general. The
configuration without a learned library uses the host's solve path with one
backend. It differs from the evaluation path in two further respects: the
keyword extractor has no library keywords to refer to, and a backend selector
takes the place of the skill selector.

\begin{table}[H]
\caption{Size of the libraries.}
\label{tab:libdiag}
\centering
\small
\begin{tabular}{lccc}
\toprule
Library & Skills & Built from one problem (\%) & Characters of skill text \\
\midrule
Three families agree & 157 & 89.2 & 1{,}736{,}590 \\
Two families agree & 172 & 89.5 & 1{,}769{,}709 \\
Majority vote & 184 & 90.2 & 2{,}038{,}275 \\
Execution success & 209 & 90.0 & 2{,}337{,}305 \\
Reference labels & 165 & 90.9 & 1{,}842{,}439 \\
\midrule
Host's released library & 56 & -- & 645{,}303 \\
\bottomrule
\end{tabular}

\end{table}

\section{Cost}
\label{app:cost}

\Cref{tab:cost} gives the number of model calls and tokens per problem and the
median time per problem. Certification consists of one extraction call, one
generation call per family, and the revision calls. Its time includes the
solver runs of the resampling rules. Data verification consists of one call
per family, and its time is the sum of the three call times.

\begin{table}[H]
\caption{Model calls, tokens, and median seconds per problem.}
\label{tab:cost}
\centering
\small
\begin{tabular}{llrrrr}
\toprule
Set & Stage & Calls & Input tokens & Output tokens & Median seconds \\
\midrule
OptMATH-295 & Certification & 4.53 & 29{,}190 & 3{,}548 & 88.5 \\
OptMATH-295 & Data verification & 2.93 & 21{,}769 & 176 & 8.4 \\
OptMATH-200 & Certification & 4.68 & 20{,}609 & 3{,}513 & 89.9 \\
OptMATH-200 & Data verification & 3.02 & 9{,}606 & 152 & 6.5 \\
NANO-CO-150 & Certification & 4.37 & 25{,}161 & 4{,}051 & 97.5 \\
NANO-CO-150 & Data verification & 3.02 & 22{,}654 & 115 & 7.6 \\
NANO-CO-95 & Certification & 4.48 & 25{,}696 & 4{,}085 & 103.3 \\
NANO-CO-95 & Data verification & 3.03 & 22{,}681 & 116 & 7.6 \\
\bottomrule
\end{tabular}

\end{table}

\end{document}